\documentclass{article} 
\usepackage{nips14submit_e,times}
\usepackage{natbib}
\usepackage{graphicx,subfigure}
\usepackage{amsmath}
\usepackage{fancyvrb}
\usepackage{amssymb}
\usepackage{wrapfig}
\usepackage{multirow}
\usepackage{hyperref}
\usepackage{algorithm}
\usepackage{algorithmic}
\usepackage{amsmath,amssymb,amsthm}
\usepackage{rotating}
\usepackage{etoolbox}

\newtoggle{arxiv}
\toggletrue{arxiv}

\newtheorem{theorem}{Theorem}
\newtheorem{lemma}[theorem]{Lemma}

\theoremstyle{definition}

\newcommand{\LL}[0]{f}
\newcommand{\Ts}[0]{\mathcal{T}}
\newcommand{\vv}[0]{\mathbf{v}}
\newcommand{\vg}[0]{\mathbf{g}}

\newcommand{\RR}[0]{\mathbb{R}}

\newcommand{\rnorm}[0]{\right| }
\newcommand{\lnorm}[0]{\left| }
\newcommand{\Dtheta}[0]{\Delta \theta}

\newcommand{\hess}[0]{\mathbf{H}}
\newcommand{\es}[0]{\mathbf{e}}

\newcommand{\example}[0]{\mathbf{x}}

\newcommand{\Fbf}[0]{\mathbf{F}}

\newcommand{\Abf}[0]{\mathbf{A}}

\newcommand{\el}[1]{_{#1}}
\def\slantfrac#1#2{\kern.1em^{#1}\kern-.3em/\kern-.1em_{#2}}

\nipsfinalcopy 

\newcommand{\razvan}[1]{\textcolor{black}{#1}}
\newcommand{\cho}[1]{\textcolor{black}{#1}}

\DeclareMathOperator*{\argmin}{arg\,min}

\begin{document}

\title{Identifying and attacking the saddle point problem in high-dimensional non-convex optimization}

\author{Yann N. Dauphin\\
Universit\'e de Montr\'eal\\
\texttt{dauphiya@iro.umontreal.ca}
\And
Razvan Pascanu\\
Universit\'e de Montr\'eal\\
\texttt{r.pascanu@gmail.com} \\
\And
Caglar Gulcehre \\
Universit\'e de Montr\'eal\\
\texttt{gulcehrc@iro.umontreal.ca} \\
\And
Kyunghyun Cho \\
Universit\'e de Montr\'eal\\
\texttt{kyunghyun.cho@umontreal.ca} \\
\And
Surya Ganguli\\
Stanford University\\
\texttt{sganguli@standford.edu}\\
\And
Yoshua Bengio\\
Universit\'e de Montr\'eal, CIFAR Fellow \\
\texttt{yoshua.bengio@umontreal.ca}\\
}

\maketitle

\begin{abstract}
    A central challenge to many fields of science and engineering involves
minimizing non-convex error functions over continuous, high dimensional spaces.
Gradient descent or quasi-Newton methods are almost ubiquitously used to
perform such minimizations, and it is often thought that a main source of
difficulty for these local methods to find the global minimum is the
proliferation of local minima with much higher error than the global minimum.
Here we argue, based on results from statistical physics, random matrix theory,
neural network theory, and empirical evidence, that a deeper and more profound
difficulty originates from the proliferation of saddle points, not local
minima, especially in high dimensional problems of practical interest. Such
saddle points are surrounded by high error plateaus that can dramatically slow
down learning, and give the illusory impression of the existence of a local
minimum.  Motivated by these arguments, we propose a new approach to
second-order optimization, the saddle-free Newton method, that can rapidly
escape high dimensional saddle points, unlike gradient descent and quasi-Newton
methods.  We apply this algorithm to deep or recurrent neural network training,
and provide numerical evidence for its superior optimization performance.
\iftoggle{arxiv}{This work extends the results of \citet{Pascanu14}.}
{}
\end{abstract}

\section{Introduction}

It is often the case that our geometric intuition, derived from experience
within a low dimensional physical world, is inadequate for thinking about the
geometry of typical error surfaces in high-dimensional spaces.  To illustrate
this, consider minimizing a randomly chosen error function of a single scalar
variable, given by a single draw of a Gaussian process. \citep{Rasmussen05}
have shown that such a random error function would have many local minima and
maxima, with high probability over the choice of the function, but saddles
would occur with negligible probability.  On the other-hand, as we review
below, typical, random Gaussian error functions over $N$ scalar variables, or
dimensions, are increasingly likely to have saddle points rather than local
minima as $N$ increases.  Indeed the ratio of the number of saddle points to
local minima increases {\it exponentially} with the dimensionality $N$. 

A typical problem for both local minima and saddle-points is that they are
often surrounded by plateaus of small curvature in the error.  While gradient
descent dynamics are repelled away from a saddle point to lower error by
following directions of negative curvature, this repulsion can occur slowly due
to the plateau.  Second order methods, like the Newton method, are designed to
rapidly descend plateaus surrounding local minima by rescaling gradient steps
by the inverse eigenvalues of the Hessian matrix.  However, the Newton method
does not treat saddle points appropriately; as argued below, saddle-points
instead become {\it attractive} under the Newton dynamics.

Thus, given the proliferation of saddle points, not local minima, in high
dimensional problems, the entire theoretical justification for quasi-Newton
methods, i.e. the ability to rapidly descend to the bottom of a convex local
minimum, becomes less relevant in high dimensional non-convex optimization.  In
this work, which is an extension of the previous report \citet{Pascanu14}, 
we first want to raise awareness of this issue, and second, propose
an alternative approach to second-order optimization that aims to rapidly
escape from saddle points.  This algorithm leverages second-order curvature
information in a fundamentally different way than quasi-Newton methods, and
also, in numerical experiments, outperforms them in some high dimensional
problems involving deep or recurrent networks.

\section{The prevalence of saddle points in high dimensions}
\label{sec:theory}

Here we review arguments from disparate literatures suggesting that saddle
points, not local minima, provide a fundamental impediment to rapid high
dimensional non-convex optimization.  One line of evidence comes from
statistical physics.  \citet{Bray07,Fyodorov07} study the nature of critical
points of random Gaussian error functions on high dimensional continuous
domains using replica theory (see \citet{Parisi07} for a recent review of this
approach).

One particular result by \citet{Bray07} derives how critical points are
distributed in the $\epsilon$ vs $\alpha$ plane, where $\alpha$ is the index,
or the fraction of negative eigenvalues of the Hessian at the critical point,
and  $\epsilon$ is the error attained at the critical point. Within this plane,
critical points concentrate on a monotonically increasing curve as $\alpha$
ranges from $0$ to $1$, implying a strong correlation between the error
$\epsilon$ and the index $\alpha$: the larger the error the larger the index.
The probability of a critical point to be an $O(1)$ distance off the curve is
exponentially small in the dimensionality $N$, for large $N$. This implies that
critical points with error $\epsilon$ much larger than that of the global
minimum, are exponentially likely to be saddle points, with the fraction of
negative curvature directions being an increasing function of the error.
Conversely, all local minima, which necessarily have index $0$, are likely to
have an error very close to that of the global minimum.  Intuitively, {\em in
high dimensions, the chance that all the directions around a critical point
lead upward (positive curvature) is exponentially small} w.r.t. the number of
dimensions, unless the critical point is the global minimum or stands at an
error level close to it, i.e., it is unlikely one can find a way to go further
down.

These results may also be understood via random matrix theory.  We know that
for a large Gaussian random matrix the eigenvalue distribution follows Wigner's
famous semicircular law~\citep{Wigner58}, with both mode and mean at $0$.  The
probability of an eigenvalue to be positive or negative is thus
$\slantfrac{1}{2}$.  \citet{Bray07} showed that the eigenvalues of the Hessian
at a critical point are distributed in the same way, except that the
semicircular spectrum is shifted by an amount determined by $\epsilon$. For the
global minimum, the spectrum is shifted so far right, that all eigenvalues are
positive. As $\epsilon$ increases, the spectrum shifts to the left and accrues
more negative eigenvalues as well as a density of eigenvalues around $0$,
indicating the typical presence of plateaus surrounding saddle points at large
error.  Such plateaus would slow the convergence of first order optimization
methods, yielding the illusion of a local minimum.

The random matrix perspective also concisely and intuitively crystallizes the
striking difference between the geometry of low and high dimensional error
surfaces.  For $N=1$, an exact saddle point is a $0$--probability event as it
means randomly picking an eigenvalue of exactly $0$. As $N$ grows it becomes
exponentially unlikely to randomly pick all eigenvalues to be positive or
negative, and therefore most critical points are saddle points. 

\citet{Fyodorov07} review qualitatively similar results derived for random
error functions superimposed on a quadratic error surface.  These works
indicate that for typical, generic functions chosen from a random Gaussian
ensemble of functions, local minima with high error are exponentially rare in
the dimensionality of the problem, but saddle points with many negative and
approximate plateau directions are exponentially likely.  However, is this
result for generic error landscapes applicable to the error landscapes of
practical problems of interest?

\citet{Baldi89} analyzed the error surface of a multilayer perceptron (MLP)
with a single linear hidden layer.  Such an error surface shows only
saddle-points and \emph{no} local minima. This result is qualitatively
consistent with the observation made by \citet{Bray07}.  Indeed
\citet{Saxe-ICLR2014} analyzed the dynamics of learning in the presence of
these saddle points, and showed that they arise due to scaling symmetries in
the weight space of a deep linear MLP. These scaling symmetries enabled
\citet{Saxe-ICLR2014} to find new exact solutions to the nonlinear dynamics of
learning in deep linear networks. These learning dynamics exhibit plateaus of
high error followed by abrupt transitions to better performance.  They
qualitatively recapitulate aspects of the hierarchical development of semantic
concepts in infants \citep{Saxe-Cogsci}.

In \citep{Saad95} the dynamics of stochastic gradient descent are analyzed for
soft committee machines.  This work explores how well a student network can
learn to imitate a randomly chosen teacher network. Importantly, it was
observed that learning can go through an initial phase of \emph{being trapped
in the symmetric submanifold} of weight space. In this submanifold, the
student's hidden units compute similar functions over the distribution of
inputs.   The slow learning dynamics within this submanifold originates from
saddle point structures (caused by permutation symmetries among hidden units),
and their associated plateaus \citep{MagnusSA_98,Inoue03}.  The exit from the
plateau associated with the symmetric submanifold corresponds to the
differentiation of the student's hidden units to mimic the teacher's hidden
units.  Interestingly, this exit from the plateau is achieved by following
directions of negative curvature associated with a saddle point. 
sin directions perpendicular to the symmetric submanifold.

\citet{Mizutani10} look at the effect of negative curvature on learning and
implicitly at the effect of saddle points in the error surface.
Their findings are similar.  They show that the error
surface of a single layer MLP has saddle points where the Hessian
matrix is indefinite.

\section{Experimental validation of the prevalence of saddle points}

In this section, we experimentally test whether the theoretical predictions
presented by \citet{Bray07} for random Gaussian fields hold for neural networks.
To our knowledge, this is the first attempt to measure the relevant statistical properties of neural network 
error surfaces and to test if the theory developed for random Gaussian fields generalizes to such cases.

In particular, we are interested in how the critical points of a single layer MLP
are distributed in the $\epsilon$--$\alpha$ plane, and how the eigenvalues of the Hessian
matrix at these critical points are distributed. We used a small MLP 
 trained on a down-sampled version of MNIST and CIFAR-10. Newton
method was used to identify critical points of the error function. The results
are in Fig.~\ref{fig:index}. 
\iftoggle{arxiv}{
More details about the setup are provided in Appendix~\ref{sec:apx_emp}.
}{
More details about the setup are provided in the supplementary material. 
}
 
\begin{figure}
\centering
\begin{minipage}{0.49\textwidth}
    \centering
    MNIST

\begin{minipage}{0.49\textwidth}
    \includegraphics[width=1.\textwidth,clip=true, trim=1cm 0cm 1cm 0cm]{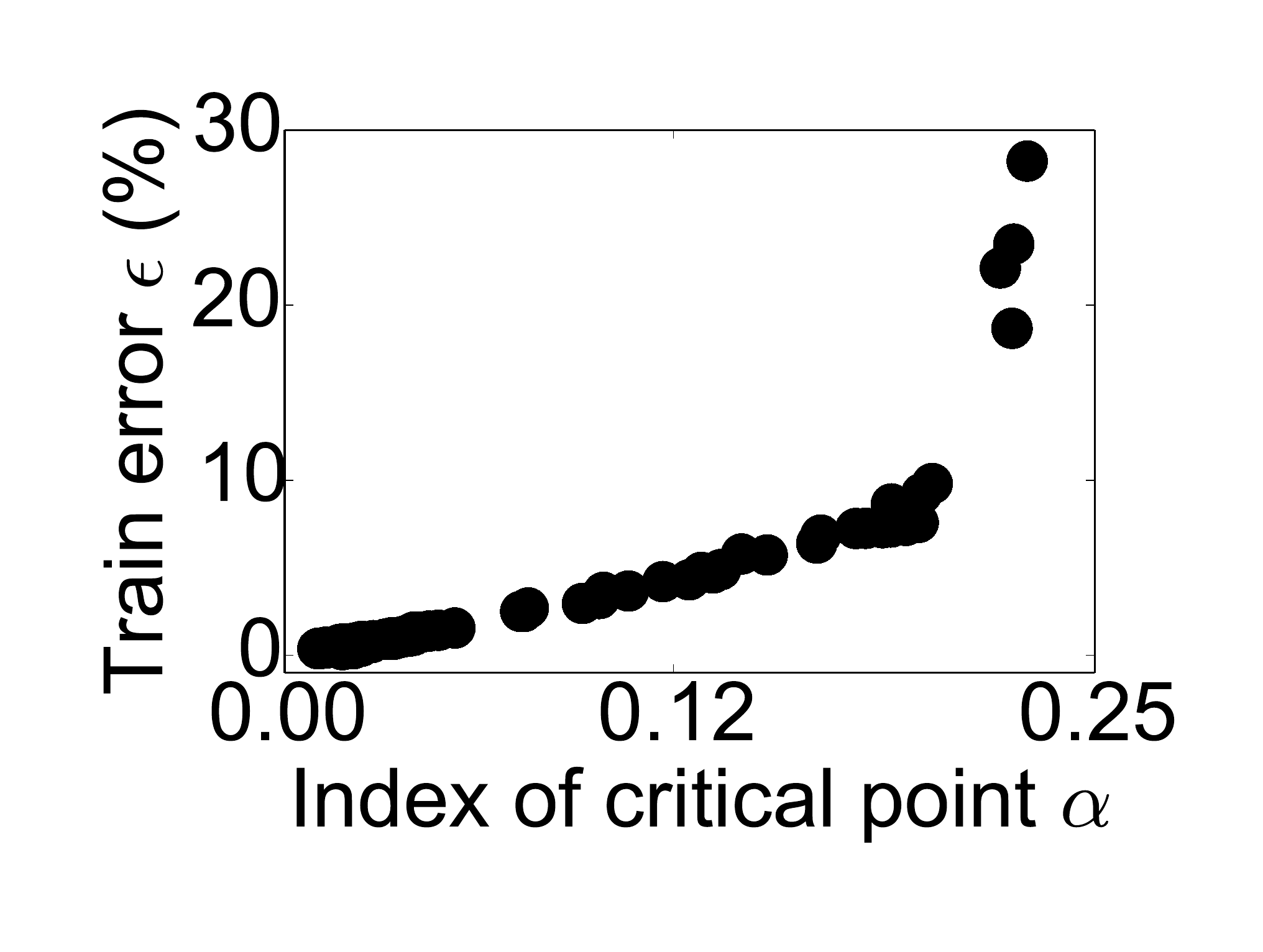}
\end{minipage}
\hfill
\begin{minipage}{0.49\textwidth}
    \includegraphics[width=1.\textwidth,clip=true, trim=1.5cm 0cm 0cm 0cm]{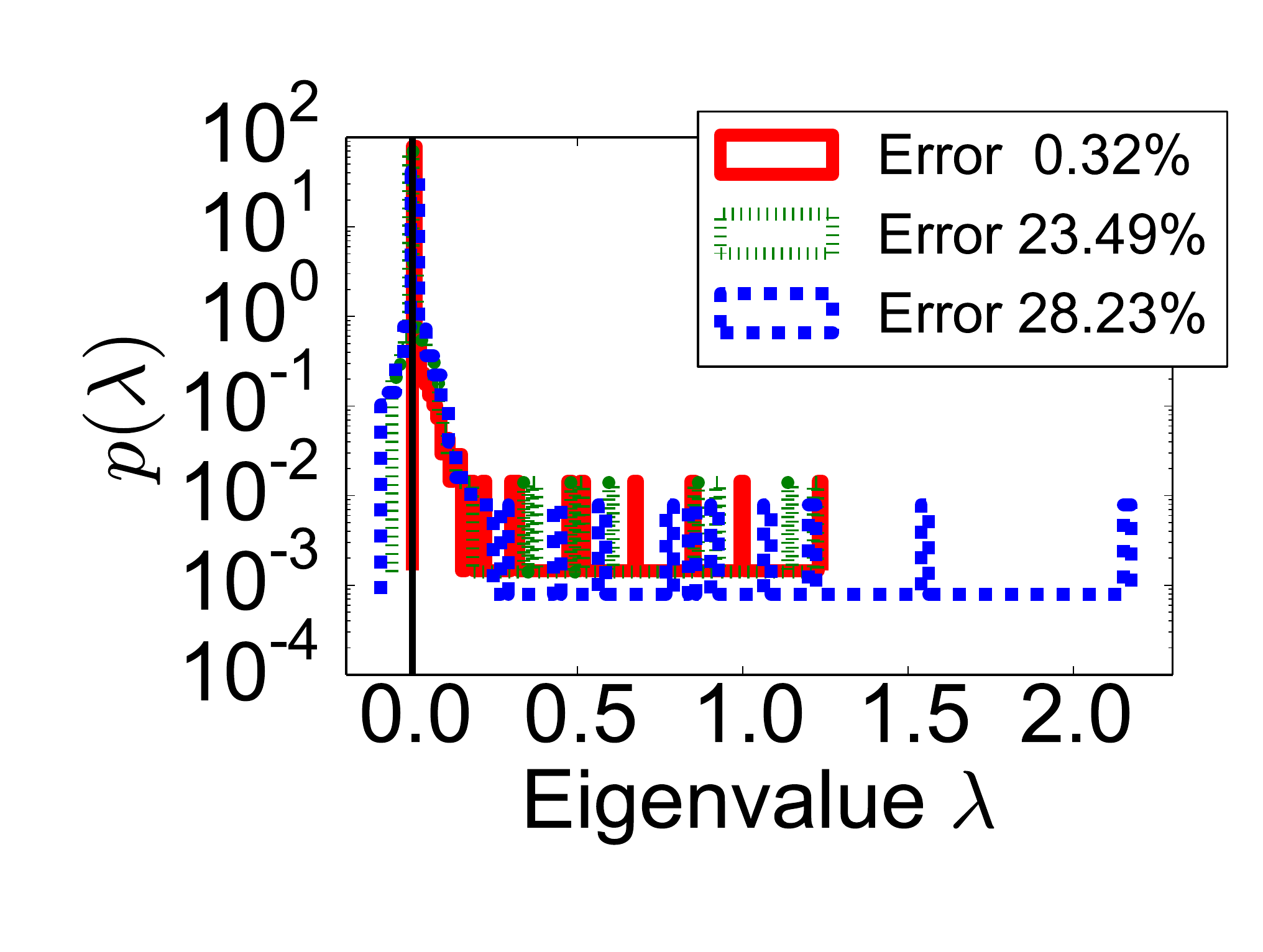}
\end{minipage}

\begin{minipage}{0.49\textwidth}
\centering
(a)
\end{minipage}
\hfill
\begin{minipage}{0.49\textwidth}
\centering
(b)
\end{minipage}
\end{minipage}
\hfill
\begin{minipage}{0.49\textwidth}
    \centering
    CIFAR-10

\begin{minipage}{0.49\textwidth}
    \includegraphics[width=1.\textwidth,clip=true, trim=1cm 0cm 1cm 0cm]{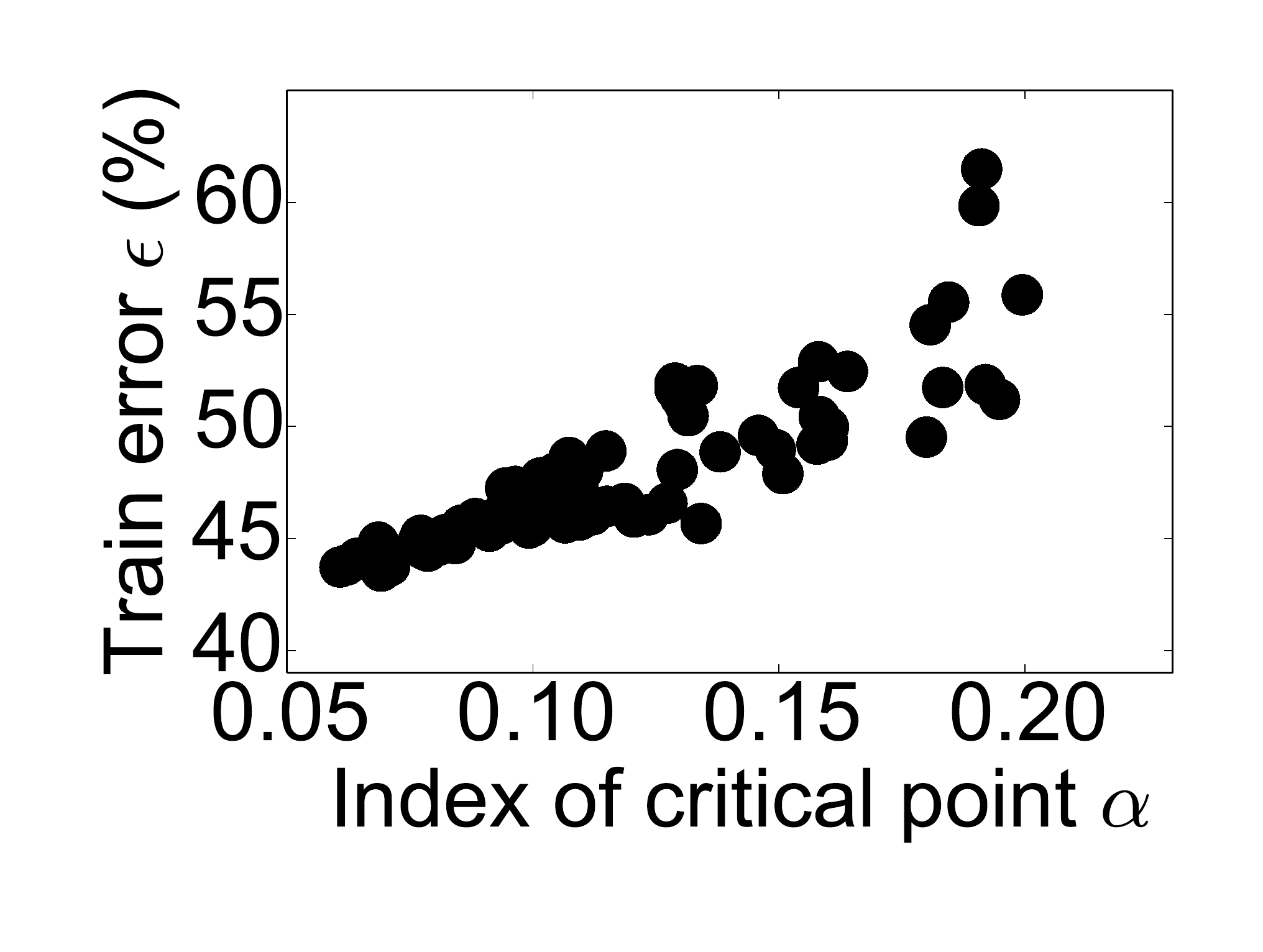}
\end{minipage}
\hfill
\begin{minipage}{0.49\textwidth}
    \includegraphics[width=1.\textwidth,clip=true, trim=1.5cm 0cm 0cm 0cm]{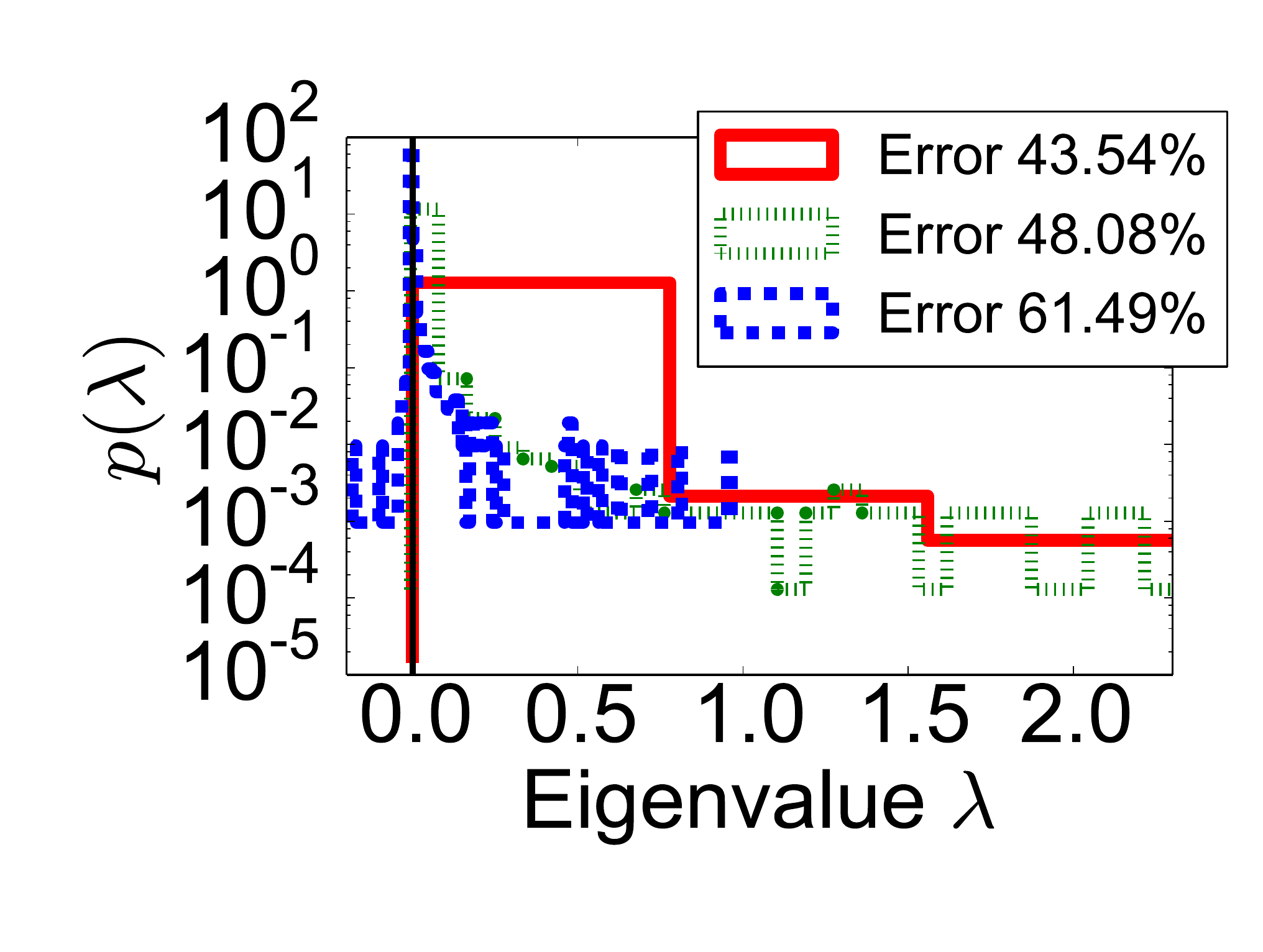}
\end{minipage}

\begin{minipage}{0.49\textwidth}
\centering
(c)
\end{minipage}
\hfill
\begin{minipage}{0.49\textwidth}
\centering
(d)
\end{minipage}
\end{minipage}
    \caption{
        (a) and (c) show how critical points are distributed in the $\epsilon$--$\alpha$
        plane. Note that they concentrate along a monotonically increasing
        curve.  (b) and (d) plot the distributions of eigenvalues of the
        Hessian at three different critical points.  Note that the y axes are 
        in logarithmic scale. \razvan{The vertical lines in (b) and (d) depict 
	the position of 0.}
    }
    \label{fig:index} 
    \vspace{-3mm}
\end{figure}

This empirical test confirms that the observations by
\citet{Bray07} qualitatively hold for neural networks.
Critical points concentrate along a monotonically increasing 
curve in the $\epsilon$--$\alpha$ plane. Thus the prevalence of high error saddle points
do indeed pose a severe problem for training neural networks. While the eigenvalues do
not seem to be exactly distributed according to the semicircular law,  their 
distribution does shift to the left as the error increases.
The large mode at 0 indicates that there is a plateau around any critical 
point of the error function of a neural network.

\section{Dynamics of optimization algorithms near saddle points}
\label{sec:optim_dynam}

Given the prevalence of saddle points, it is important to understand how various optimization algorithms behave near them. Let us focus on non-degenerate saddle points for which the Hessian is not
singular. These critical points can be locally analyzed by re-parameterizing the
function according to Morse's lemma below (see chapter 7.3, Theorem 7.16 in
\citet{callahan2010advanced} 
\iftoggle{arxiv}{
or Appendix~\ref{sec:apx_reparam}:
}{
or the supplementary material for details):
}

\begin{equation}
\label{eq:new_system_coord}
\LL(\theta^* + \Dtheta) = \LL(\theta^*) + \frac{1}{2} \sum_{i=1}^{n_\theta}
{\lambda\el i} \Delta \vv_i^2,
\end{equation}
where $\lambda\el i$ represents the $i$th eigenvalue of the Hessian, and 
$\Delta\vv_i$ are the new parameters of the model corresponding to motion along 
the eigenvectors $\es\el i$ of the Hessian of $\LL$ at $\theta^*$.

A step of the \emph{gradient descent} method always points in the right direction close
to a saddle point (SGD in Fig.~\ref{fig:saddle_methods}).  If an eigenvalue
${\lambda\el i}$ is positive (negative), then the step moves toward (away) from $\theta^*$ along $\Delta\vv_i$ because the
restriction of $\LL$ to the corresponding eigenvector direction $\Delta\vv_i$, achieves a 
minimum (maximum) at $\theta^*$.  The drawback of the gradient descent method is not the direction, but the \emph{size} of
the step along each eigenvector. The step, along any direction
$\es\el i$, is given by $-{\lambda\el i} \Delta \vv_i$, and so small steps are taken in directions corresponding to eigenvalues of 
small absolute value. 

\begin{figure}[t]
\centering
    \begin{minipage}{.64\textwidth}
        \centering
        \begin{minipage}{0.49\textwidth}
            \centering
            \includegraphics[width=1\textwidth,clip=true, trim=2.5cm 2.5cm 3cm 3.7cm]{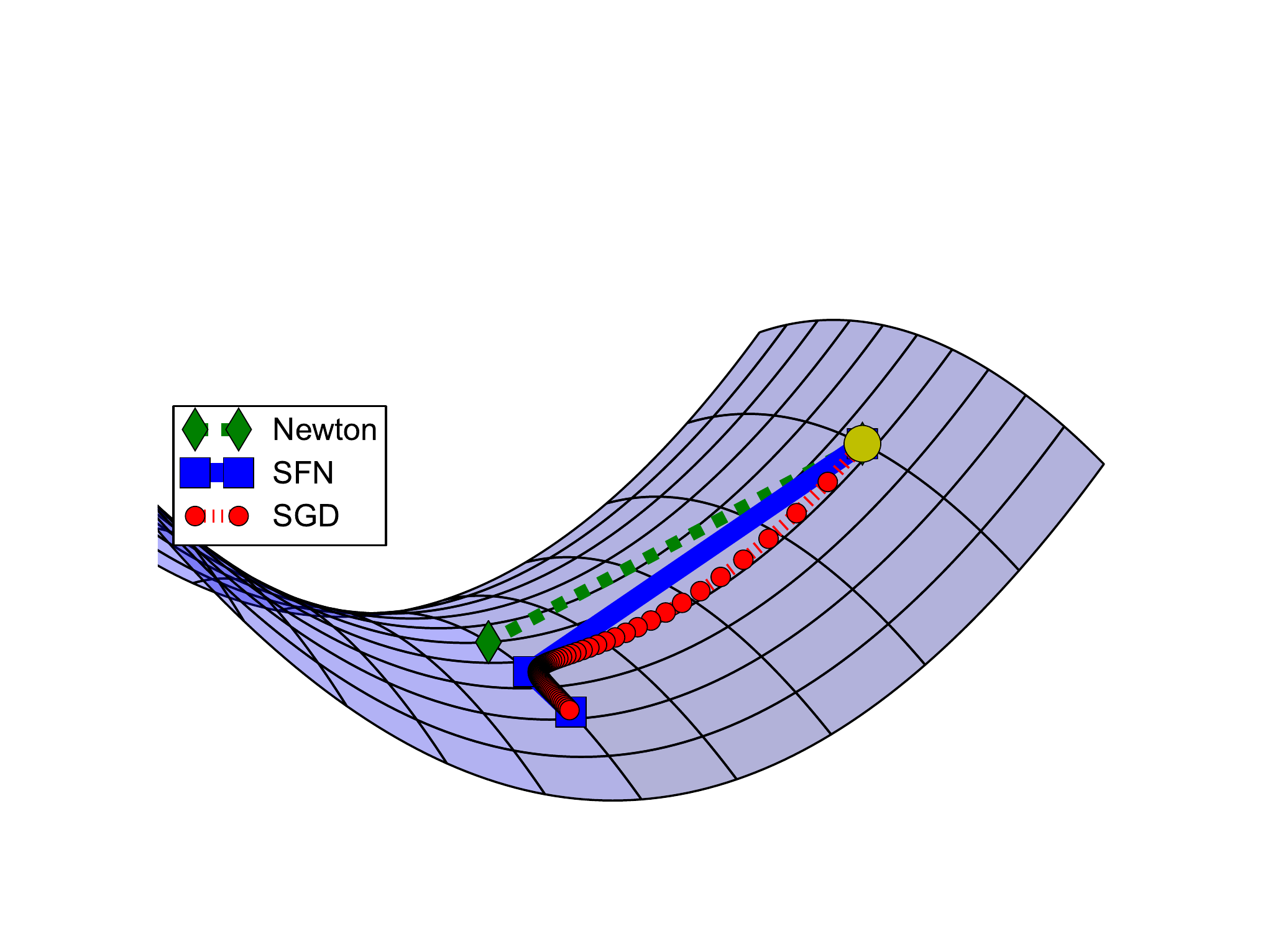}
        \end{minipage}
    \hfill
        \begin{minipage}{0.49\textwidth}
            \centering
            \includegraphics[width=1\textwidth,clip=true, trim=2.5cm 2.5cm 3cm 0cm]{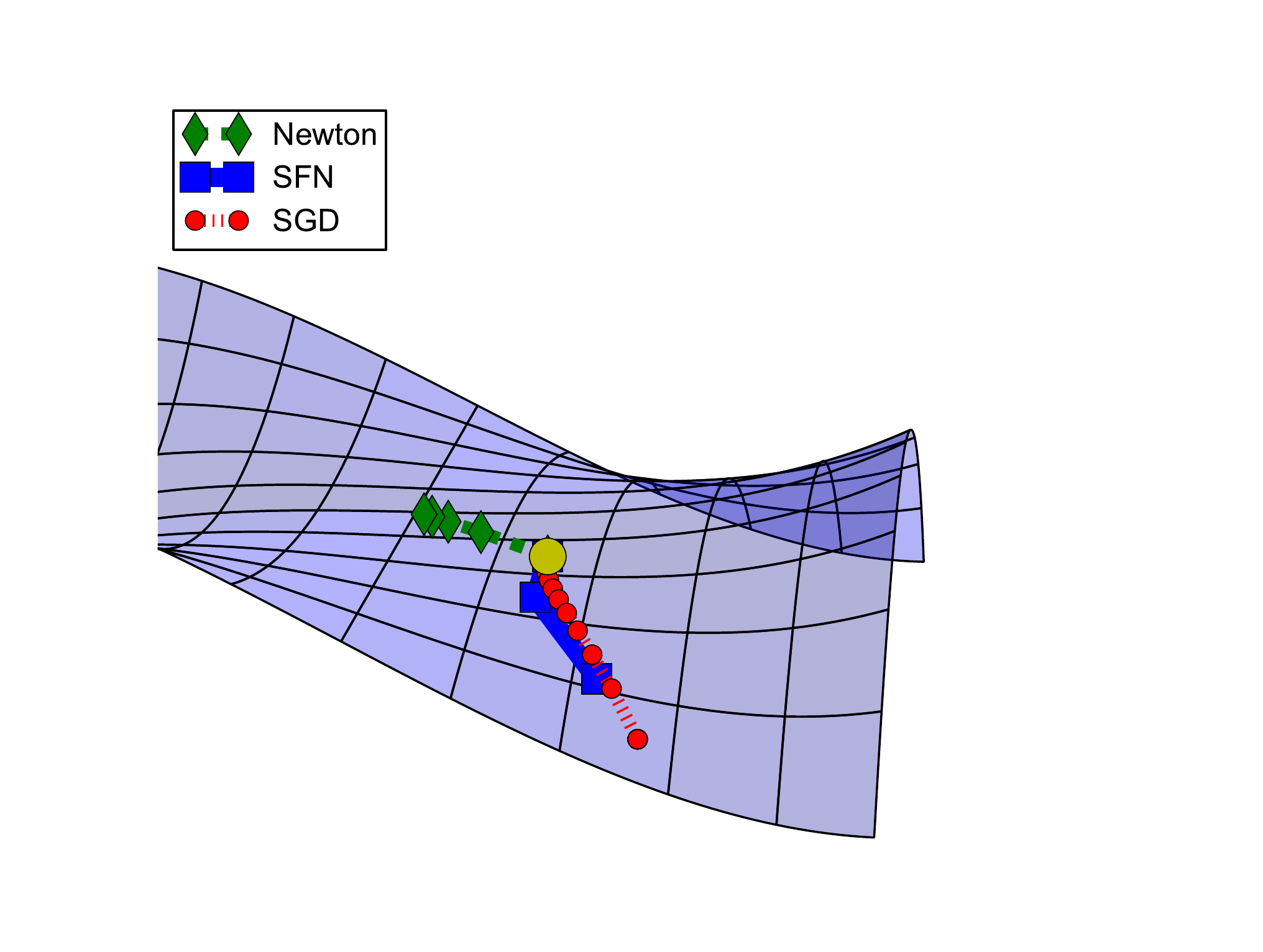}
        \end{minipage}

        \begin{minipage}{0.49\textwidth}
            \centering
	(a)
        \end{minipage}
        \hfill
        \begin{minipage}{0.49\textwidth}
            \centering
    (b)
        \end{minipage}
    \end{minipage}
    \hfill
    \begin{minipage}{0.35\textwidth}
        \caption{Behaviors of different optimization methods near a saddle point for (a) classical saddle structure \mbox{$5x^2 - y^2$}; (b) monkey saddle structure \mbox{$x^3 - 3xy^2$}. 
The yellow 
dot indicates the starting point. SFN stands for the saddle-free Newton method we proposed.
}
    \label{fig:saddle_methods} 
\end{minipage}
\end{figure}

The \emph{Newton method} solves the slowness problem by rescaling the
gradients in each direction with the inverse of the corresponding eigenvalue, yielding 
the step  $-\Delta \vv_i$.  However, this approach can
result in moving in the wrong direction. Specifically, if an eigenvalue is
negative, the Newton step moves along the eigenvector in a direction {\it opposite} 
to the gradient descent step, and thus moves in the direction of {\it increasing} 
error {\it towards} $\theta^*$. Since it also moves towards $\theta^*$ along eigen-directions
with positive eigenvalue, the saddle point $\theta^*$ becomes an \emph{attractor} for the
Newton method (see Fig.~\ref{fig:saddle_methods}).  This justifies using the Newton method 
to find critical points of any index in Fig. ~\ref{fig:index}.

A \emph{trust region} approach is a practical implementation of second order
methods for non-convex problems. In one such method, the Hessian is damped to remove negative
curvature by adding a constant $\alpha$ to its
diagonal, which is equivalent to adding $\alpha$ to each of its eigenvalues.
Rescaling the gradient by the inverse of the modified eigenvalues $\lambda\el i + \alpha$ yields the step $-
\left(\slantfrac{{\lambda\el i}}{{\lambda\el i} + \alpha}\right) \Delta \vv_i$.
To ensure descent along every eigen-direction, 
one must increase the damping coefficient $\alpha$ enough so that $\lambda\el{min} + \alpha > 0$
even for the most negative eigenvalue $\lambda\el{min}$.
Therefore, the drawback is again a potentially small step size in many eigen-directions incurred by 
a large damping factor $\alpha$. 

Besides damping, another approach to deal with negative curvature 
is to ignore them. This can be done regardless of the approximation strategy
used for the Newton method such as a truncated Newton method or a BFGS
approximation (see \citet{NumOptBook} chapters 4 and 7). However, such algorithms
cannot escape saddle points, as they ignore the very directions of negative curvature
that must be followed to achieve escape.

Natural gradient descent is a first order method that relies on the curvature of
the parameter manifold. That is, natural gradient descent takes a step that
induces a constant change in the behaviour of the model as measured by the
KL-divergence between the model before and after taking the step. The resulting
algorithm is similar to the Newton method, except that it relies on the Fisher
Information matrix $\Fbf$. 

It is argued by \citet{MagnusSA_98, Inoue03} that natural gradient descent can
address certain saddle point structures effectively.  Specifically, it can
resolve those saddle points arising from having units behaving very similarly.
\citet{Mizutani10}, however, argue that natural gradient descent also
suffers with negative curvature. One particular known issue is the
over-realizable regime, where around the stationary solution $\theta^*$, the
Fisher matrix is rank-deficient. Numerically, this means that the Gauss-Newton
direction can be orthogonal to the gradient at some distant point from
$\theta^*$~\citep{Mizutani10}, causing optimization to converge to some
non-stationary point. Another weakness is that 
the difference $\mathbf{S}$ between the Hessian and the Fisher Information Matrix
can be large near certain saddle points that exhibit strong negative
curvature. This means that the landscape close to these critical points may be
dominated by $\mathbf{S}$, meaning that the rescaling provided by $\Fbf^{-1}$ is
not optimal in all directions. 

The same is true for TONGA~\citep{RouxMB07}, an algorithm similar to natural
gradient descent. It uses the covariance of the gradients as the
rescaling factor. As these gradients vanish approaching a critical point, their
covariance will result in much larger steps than needed near critical points. 

\section{Generalized trust region methods}

In order to attack the saddle point problem, and overcome the deficiencies of the above methods, 
we will define a class of \emph{generalized trust region methods}, and search for an algorithm within
this space.  This class involves a straightforward extension of classical trust region methods via 
two simple changes: (1) We allow the minimization of a first-order Taylor expansion of the function instead of always
relying on a second-order Taylor expansion as is typically done in trust region
methods, and (2) we replace the constraint on the norm of the step $\Dtheta$ by a
constraint on the distance between $\theta$ and $\theta + \Dtheta$.  Thus the choice of
distance function and Taylor expansion order specifies an algorithm. If we define $\Ts_k(\LL, \theta,
\Dtheta)$ to indicate the $k$-th order Taylor series expansion of $\LL$ around
$\theta$ evaluated at $\theta + \Dtheta$, then we can summarize a generalized
trust region method as:
\begin{equation}
    \begin{aligned}
        \Delta \theta & = 
            \argmin_{\Delta \theta} 
                    \Ts_k\{\LL, \theta, \Delta\theta\} &
                    \text{ with } k \in \{1, 2\} \\
        & \text{s. t. } d(\theta, \theta + \Delta \theta) \leq \Delta. & 
    \end{aligned} 
    \label{eq:gtrm}
\end{equation}
For example, the $\alpha$-damped Newton method described above
arises as a special case with $k=2$ and $d(\theta, \theta + \Delta \theta) = || \Delta \theta ||_2^2$, where $\alpha$ 
is implicitly a function of $\Delta$.  

\begin{wrapfigure}{L}{0.5\textwidth}
    \vspace{-3mm}
    \begin{minipage}{0.5\textwidth}
        \begin{algorithm}[H]
    \caption{Approximate saddle-free Newton}
    \label{alg:sfn}
    \begin{algorithmic}
        \REQUIRE Function $f(\theta)$ to minimize
        
        \FOR{$i = 1 \to M$}
        \STATE ${\bf \mathbf{V}} \gets \text{$k$ Lanczos vectors of $\frac{\partial^2
        f}{\partial \theta^2}$}$
	\STATE $\hat{f}(\alpha) \gets g(\theta + \mathbf{V}\alpha)$
        \STATE $|\hat{\bf H}| \gets \left|\frac{\partial^2   \hat{f}}{\partial \alpha^2}\right|$ by using an eigen decomposition of $\hat{\bf H}$

        \FOR{$j = 1 \to m$}
        \STATE $\vg \gets -\frac{\partial \hat{f}}{\partial \alpha}$
        \STATE $\lambda \gets \argmin_\lambda \hat{f}(\vg(|\hat{\bf H}| + \lambda {\bf I})^{-1})$
        \STATE $\theta \gets \theta + \vg(|\hat{\bf H}| + \lambda {\bf I})^{-1}{\bf V}$
        \ENDFOR
        
        \ENDFOR
    \end{algorithmic}
    \end{algorithm}
\end{minipage}
\end{wrapfigure}

\section{Attacking the saddle point problem}
\label{sec:addressing}

We now search for a solution to the saddle-point problem within the family of generalized 
trust region methods.  In particular, the analysis of optimization algorithms near 
saddle points discussed in Sec.~\ref{sec:optim_dynam} suggests a simple heuristic solution: rescale 
the gradient along each eigen-direction $\es\el i$ by $\slantfrac{1}{|\lambda\el i|}$.  This achieves the same
optimal rescaling as the Newton method, while preserving the sign of the
gradient, thereby turning saddle points into repellers, not attractors, of the learning
dynamics. The idea of taking the absolute value of the eigenvalues of the
Hessian was suggested before. See, for example, \citep[chapter
3.4]{NumOptBook} or \citet[chapter 4.1]{Murray10}. However, we are not
aware of any proper justification of this algorithm or even a 
detailed exploration (empirical or otherwise) of this idea. One cannot simply 
replace $\hess$ by $|\hess|$, where $|\hess|$ is the matrix
obtained by taking the absolute value of each eigenvalue of $\hess$, without
proper justification. For instance, one obvious question arises: are we still
optimizing the same function?  While we might be able to argue that this heuristic 
modification does the
right thing near critical points, is it still the right thing far away from the
critical points?  Here we show this heuristic solution arises naturally from our generalized trust region approach.

Unlike classical trust region approaches, we consider minimizing a first-order Taylor expansion 
of the loss ($k=1$ in Eq.~\eqref{eq:gtrm}).  This means
that the curvature information has to come from the constraint by picking a
suitable distance measure $d$ (see Eq.~\eqref{eq:gtrm}). Since the minimum of the first order approximation of $\LL$ is at infinity, we
know that this optimization dynamics will always jump to
the border of the trust region. So we must ask how far from $\theta$ can we trust the first order approximation
of $\LL$? One answer is to bound the discrepancy between the first and second order Taylor expansions of $\LL$ by imposing the following constraint:
\begin{equation}
\label{eq:constr}
d(\theta, \theta + \Delta \theta) = \lnorm \LL(\theta) + \nabla \LL \Dtheta + 
\frac{1}{2}\Dtheta^\top \hess \Dtheta - \LL(\theta) - \nabla \LL \Dtheta \rnorm =
\frac{1}{2}\lnorm \Dtheta^\top \hess \Dtheta\rnorm \leq \Delta,
\end{equation}

where $\nabla \LL$ is the partial derivative of $\LL$ with respect to $\theta$
and $\Delta \in \RR$ is some small value that indicates how much
discrepancy we are willing to accept. Note that the distance measure $d$ takes
into account the curvature of the function. 

Eq.~\eqref{eq:constr} is not easy to solve for $\Dtheta$ in more than
one dimension. Alternatively, one could take the square of the distance, but this would yield 
an optimization problem with a constraint that is quartic in $\Dtheta$, and therefore also difficult to solve.  We circumvent
these difficulties through a Lemma:

\begin{samepage}
\begin{lemma}
\label{lemma:constraint}
Let $\Abf$ be a nonsingular square matrix in $\RR^{n} \times \RR^{n}$, and
$\example \in \RR^n$ be some vector. Then it holds that $|\example^\top \Abf
\example | \leq \example^\top |\Abf| \example$, where $|\Abf|$ is the matrix
obtained by taking the absolute value of each of the eigenvalues of $\Abf$.
\end{lemma}

\begin{proof}
\iftoggle{arxiv}{
    See Appendix~\ref{sec:apx_proof} for the proof.
}{
    See the supplementary material for the proof.
}
\end{proof}
\end{samepage}

Instead of the originally proposed distance measure in Eq.~\eqref{eq:constr}, we approximate the distance
by its upper bound $\Dtheta |\hess| \Dtheta$ based on
Lemma~\ref{lemma:constraint}. This results in the following generalized trust
region method:
\begin{equation}
    \begin{aligned}
        \Delta \theta & = 
            \argmin_{\Delta \theta} 
                    \LL(\theta) + \nabla \LL \Delta \theta \\
        & \text{s. t. } \Delta \theta^\top |\hess|\Delta \theta \leq \Delta .
    \end{aligned} 
\end{equation}

Note that as discussed before, we can replace the inequality constraint with an
equality one, as the first order approximation of $\LL$ has a minimum at
infinity and the algorithm always jumps to the border of the trust region.
Similar to \citep{Pascanu+Bengio-ICLR2014}, we use Lagrange multipliers to
obtain the solution of this constrained optimization. This gives (up to a scalar
that we fold into the learning rate) a step of the form:
\begin{equation}
\Delta \theta = -\nabla \LL |\hess|^{-1}
\end{equation}

\begin{figure}[t]
    \centering
    \begin{minipage}{0.04\textwidth}
	\centering 
        \begin{sideways}
            MNIST
        \end{sideways}
    \end{minipage}
    \hfill
    \begin{minipage}{0.95\textwidth}
        \begin{minipage}{0.32\textwidth}
            \centering
            \includegraphics[width=1.\columnwidth, clip=true, trim=0cm 0cm 0cm 0cm]{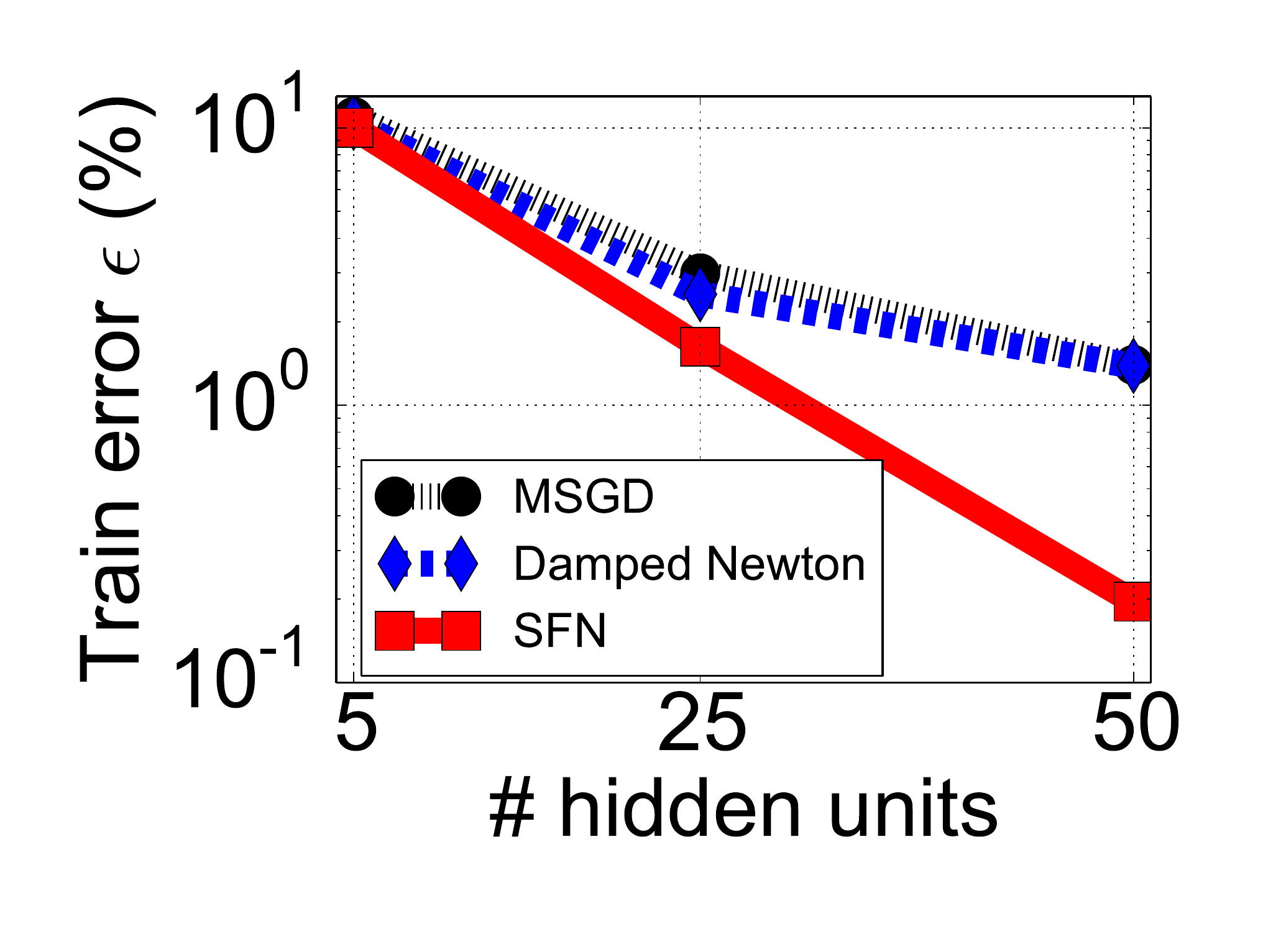}
        \end{minipage}
        \hfill
        \begin{minipage}{0.32\textwidth}
            \centering
            \includegraphics[width=1.\columnwidth, clip=true, trim=0cm 0cm 0cm 0cm]{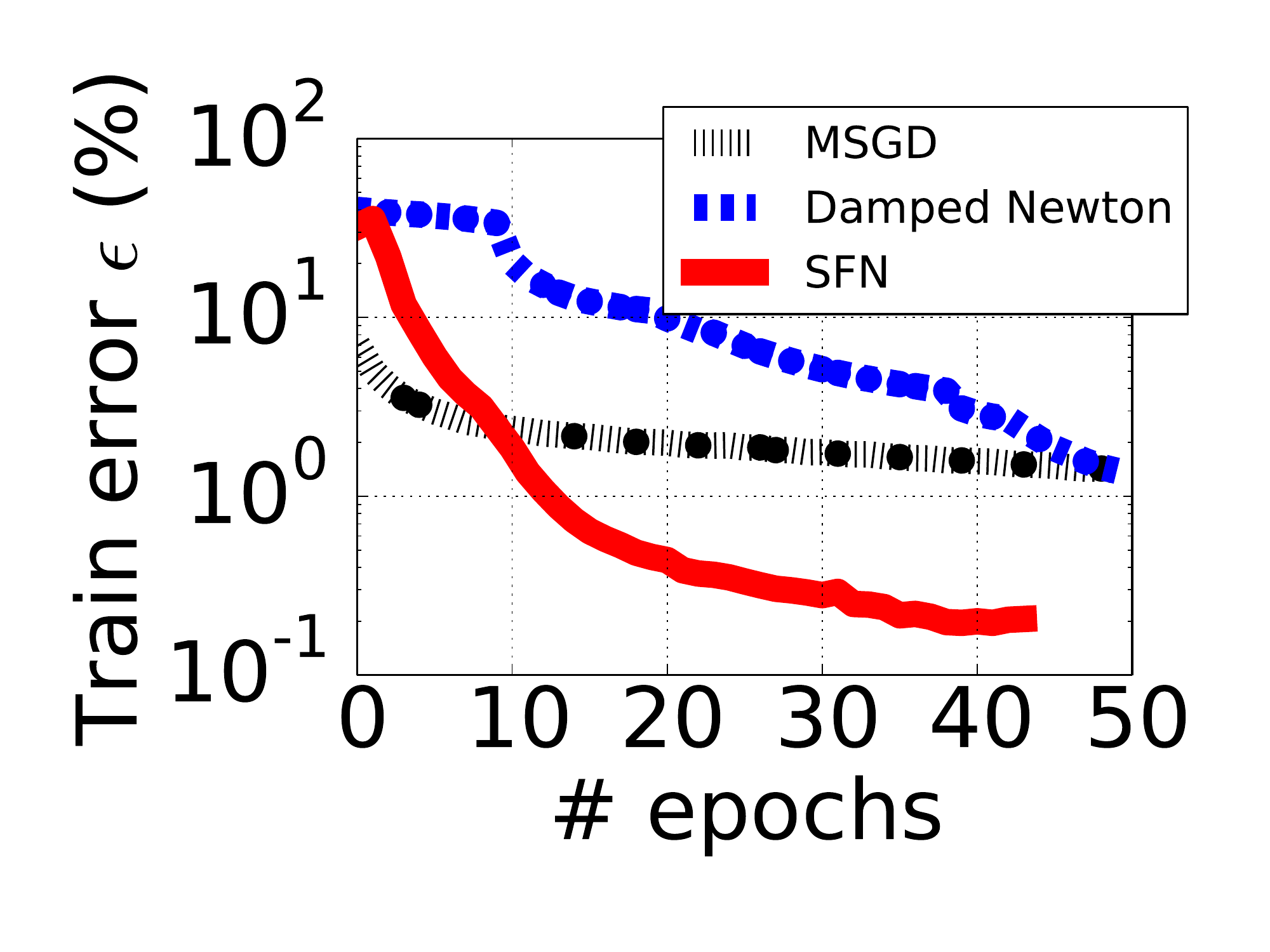}
        \end{minipage}
        \begin{minipage}{0.32\textwidth}
            \centering
            \includegraphics[width=1.\columnwidth, clip=true, trim=0cm 0cm 0cm 0cm]{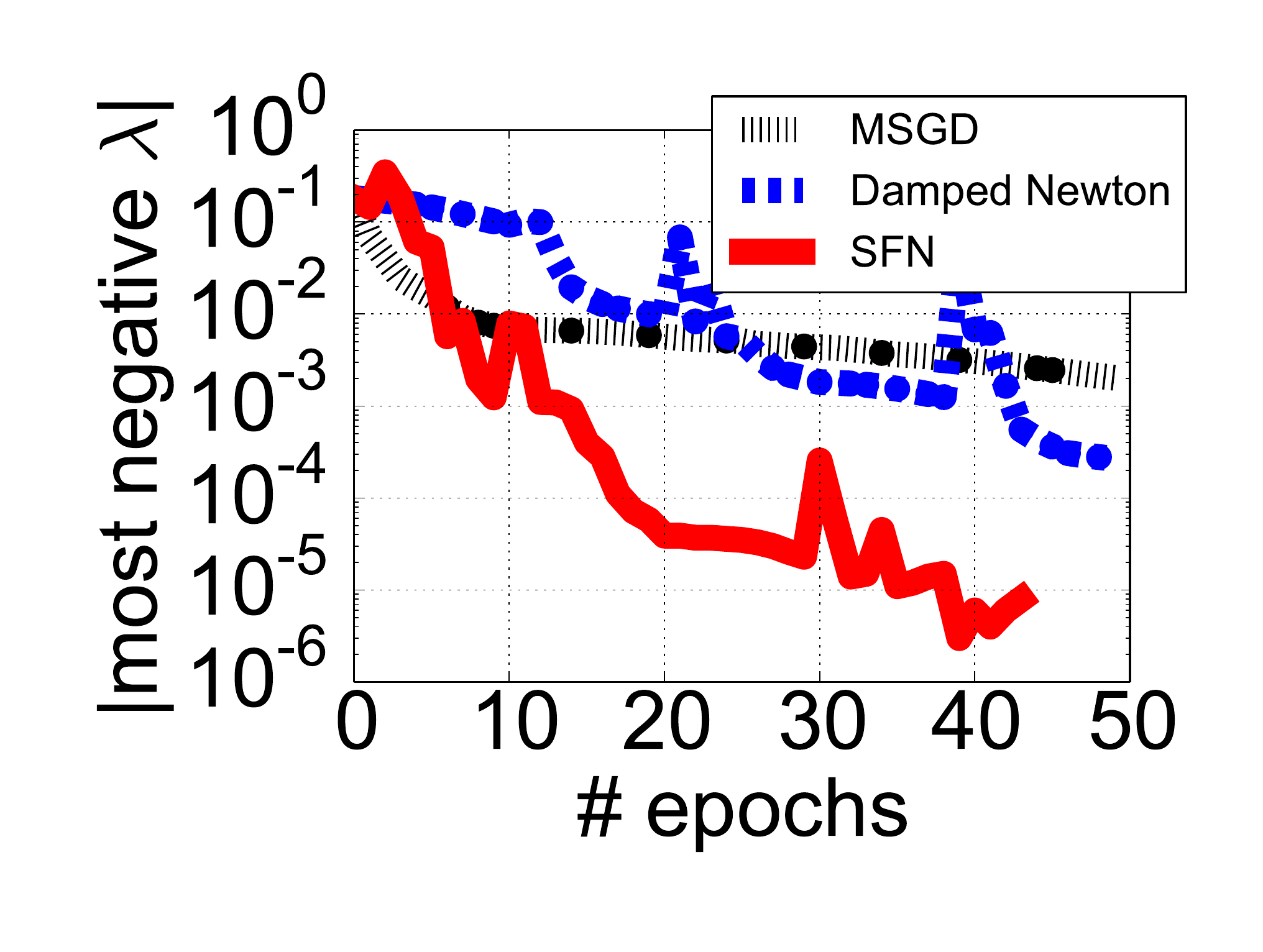}
        \end{minipage}

        \begin{minipage}{0.32\textwidth}
            \centering
            (a) 
        \end{minipage}
        \hfill
        \begin{minipage}{0.32\textwidth}
            \centering
            (b) 
        \end{minipage}
        \begin{minipage}{0.32\textwidth}
            \centering
            (c) 
        \end{minipage}
    \end{minipage}

    \begin{minipage}{0.04\textwidth}
        \centering
        \begin{sideways}
            CIFAR-10
        \end{sideways}
    \end{minipage}
    \hfill
    \begin{minipage}{0.95\textwidth}

        \begin{minipage}{0.32\textwidth}
            \centering
            \includegraphics[width=1.\columnwidth, clip=true, trim=0cm 0cm 0cm 0cm]{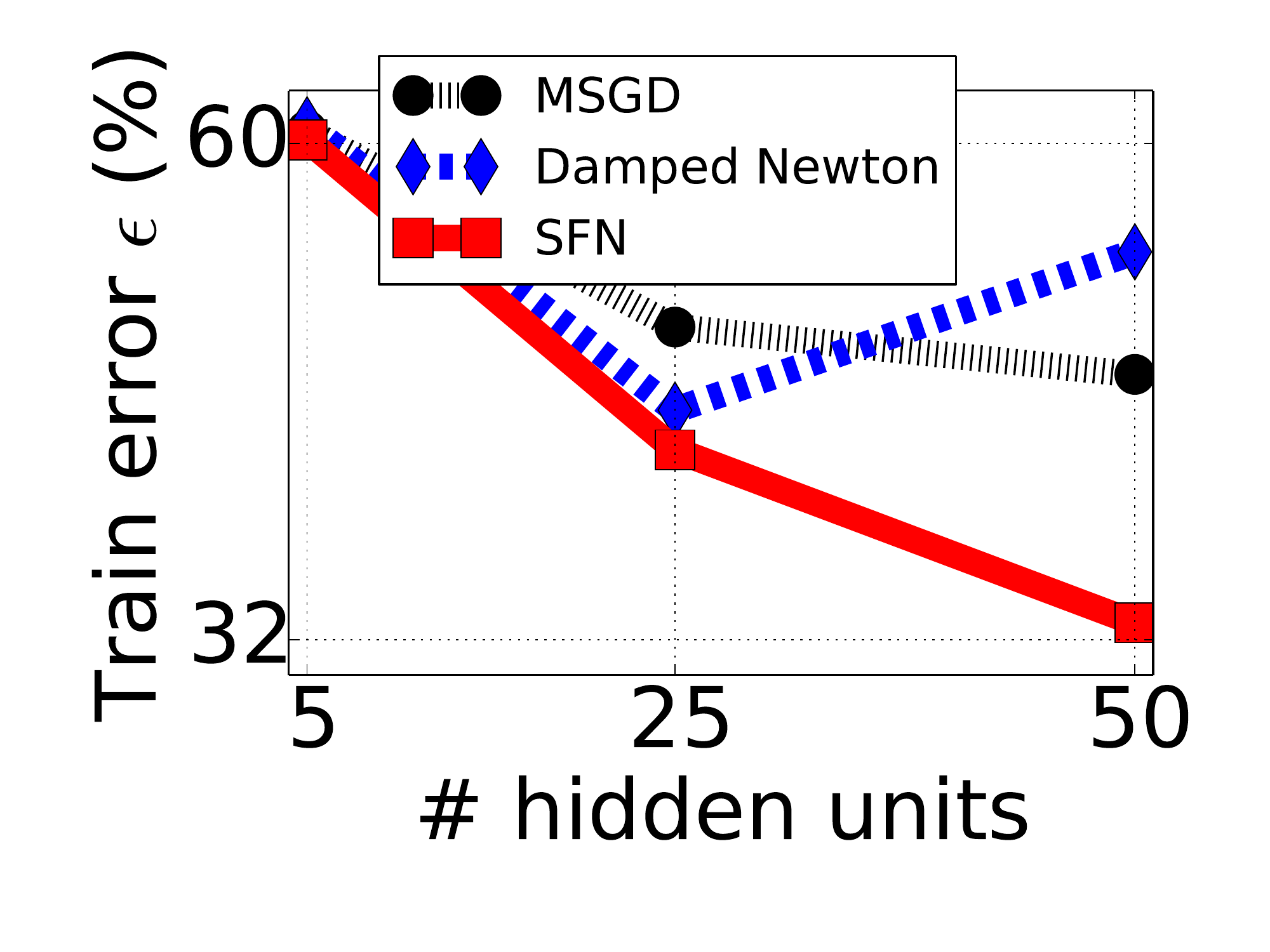}
        \end{minipage}
        \hfill
        \begin{minipage}{0.32\textwidth}
            \centering
            \includegraphics[width=1.\columnwidth, clip=true, trim=0cm 0cm 0cm 0cm]{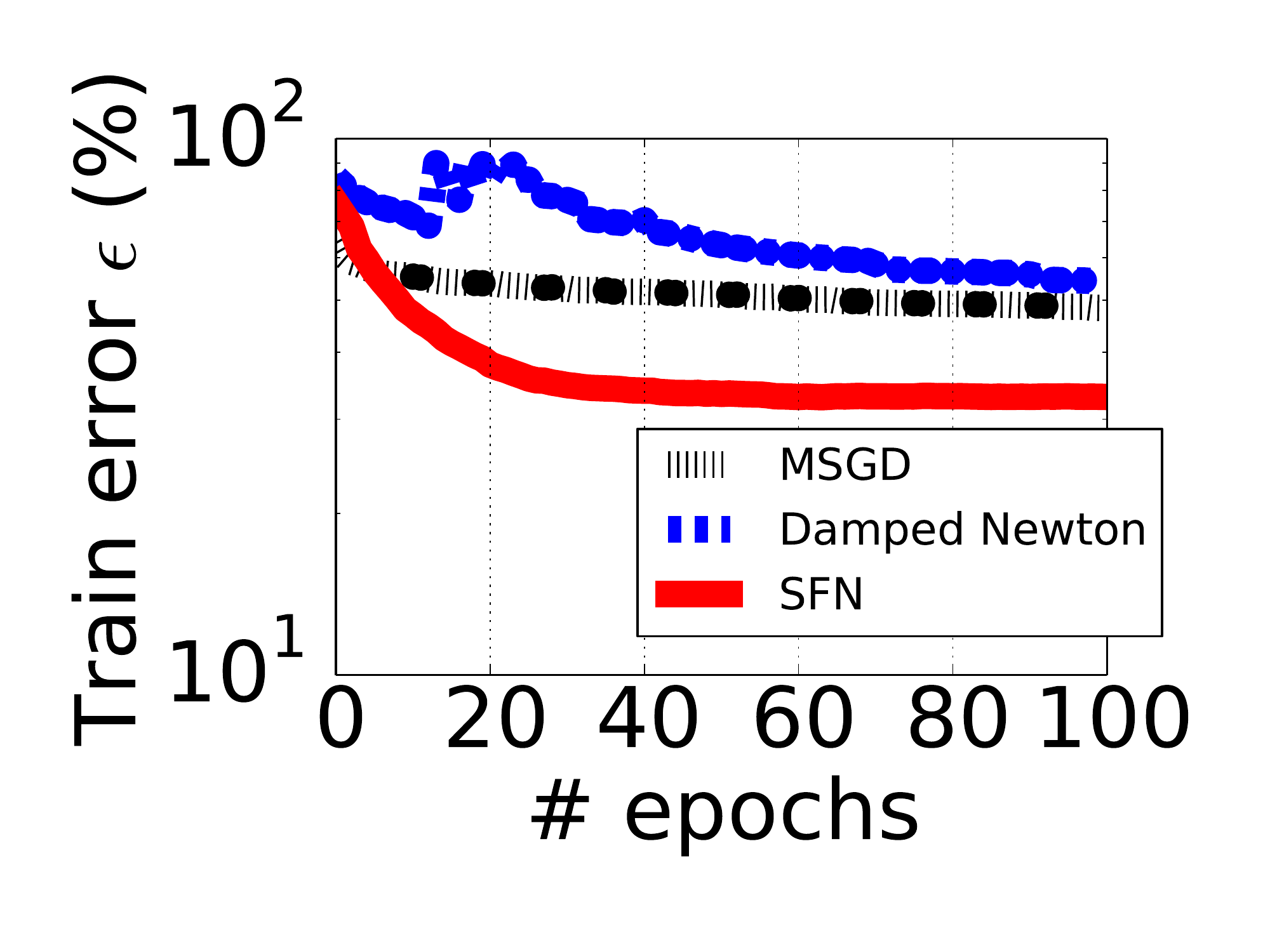}
        \end{minipage}
        \begin{minipage}{0.32\textwidth}
            \centering
            \includegraphics[width=1.\columnwidth, clip=true, trim=0cm 0cm 0cm 0cm]{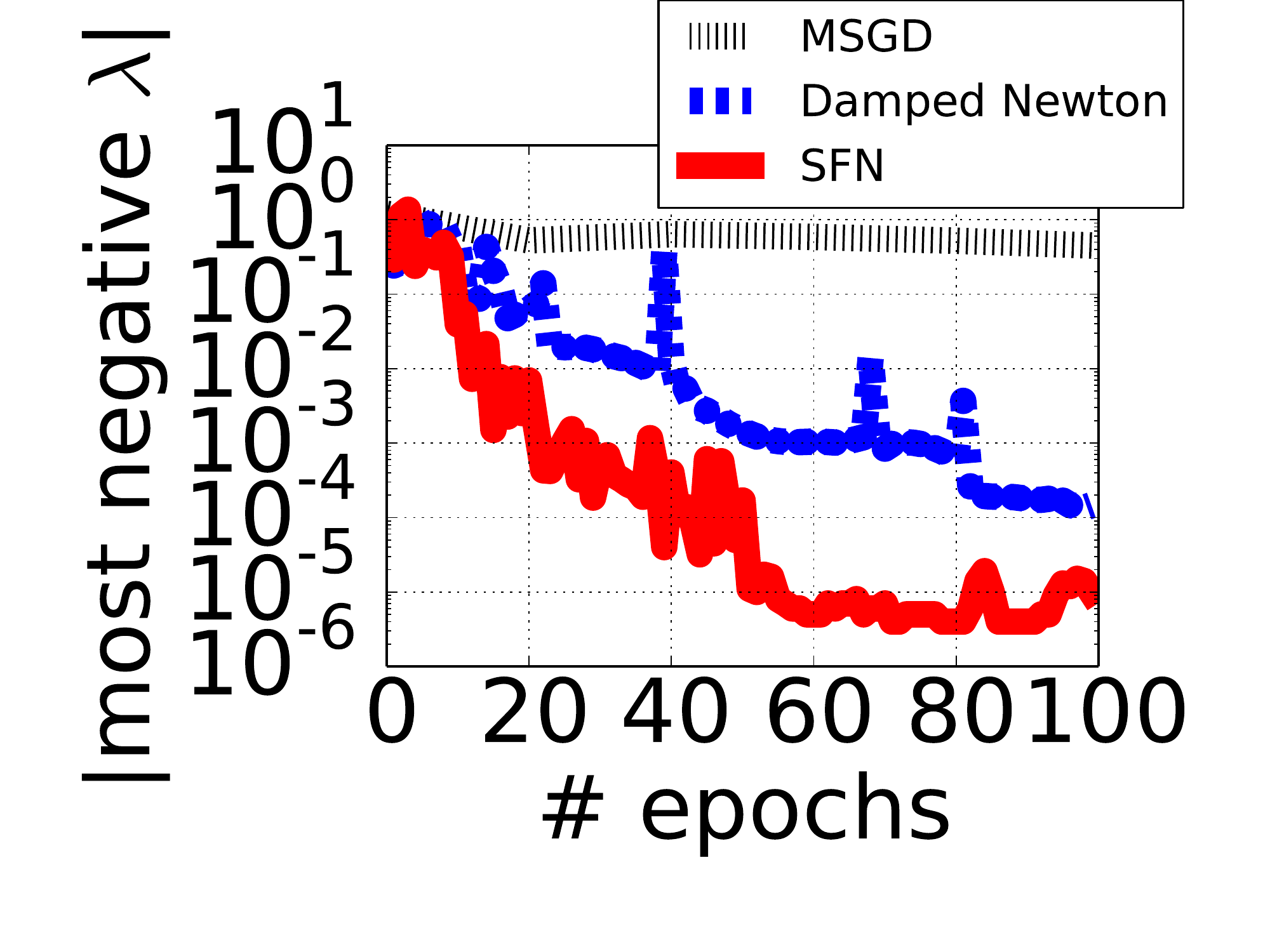}
        \end{minipage}

        \begin{minipage}{0.32\textwidth}
            \centering
            (d) 
        \end{minipage}
        \hfill
        \begin{minipage}{0.32\textwidth}
            \centering
            (e) 
        \end{minipage}
        \begin{minipage}{0.32\textwidth}
            \centering
            (f)
        \end{minipage}
    \end{minipage}
\caption{
Empirical evaluation of different optimization algorithms for a single-layer
MLP trained on the rescaled MNIST and CIFAR-10 dataset. In (a) and (d) we look at
the minimum error obtained by the different algorithms considered as a function of the model
size. (b) and (e) show the optimal training curves for the three algorithms. The error
is plotted as a function of the number of epochs. (c) and (f) track the norm of the largest 
negative eigenvalue. 
}
\label{fig:results}
    \vspace{-3mm}
\end{figure}

This algorithm, which \razvan{we call the saddle-free Newton method (SFN)}, leverages curvature information in a fundamentally different way,
to define the shape of the trust region, rather than Taylor expansion to second
order, as in classical methods.  Unlike gradient descent, it can move further
(less) in the directions of low (high) curvature.  It is identical to the Newton
method when the Hessian is positive definite, but unlike the Newton method, it
can escape saddle points.  Furthermore, unlike gradient descent, the escape is
rapid even along directions of weak negative curvature (see
Fig.~\ref{fig:saddle_methods}).

The exact implementation of this algorithm is intractable in a high dimensional
problem, because it requires the exact computation of the Hessian. Instead we use an approach
similar to Krylov subspace descent~\citep{VinyalsAISTATS12}. 
We optimize that function in a lower-dimensional Krylov subspace
$\hat{f}(\alpha) = f(\theta + \alpha {\bf V})$.  The $k$ Krylov subspace vectors
${\bf V}$ are found through Lanczos iteration of the Hessian. These vectors will
span the $k$ biggest eigenvectors of the Hessian with high-probability. This
reparametrization through $\alpha$ greatly reduces the dimensionality and allows
us to use exact saddle-free Newton in the subspace.\footnote{
    In the Krylov subspace, $\frac{\partial \hat{f}}{\partial
    \alpha} = {\bf V} \left(\frac{\partial f}{\partial \theta}\right)^\top$ and 
    $\frac{\partial^2 \hat{f}}{\partial \alpha^2} = {\bf V} \left(\frac{\partial^2
    f}{\partial \theta^2}\right) {\bf V}^\top$.
} See Alg.~\ref{alg:sfn} for the pseudocode.

\section{Experimental validation of the saddle-free Newton method}

In this section, we empirically evaluate the theory suggesting the existence of
many saddle points in high-dimensional functions by training neural
networks.

\subsection{Feedforward Neural Networks}

\subsubsection{Existence of Saddle Points in Neural Networks}

In this section, we validate the existence of saddle points in the cost function of neural
networks, and see how each of the algorithms we described earlier behaves near
them. In order to minimize the effect of any type of approximation used in the
algorithms, we train small neural networks on the scaled-down version of MNIST
and CIFAR-10, where we can compute the update directions by each algorithm
exactly. Both MNIST and CIFAR-10 were downsampled to be of size $10 \times 10$.

We compare minibatch stochastic gradient descent (MSGD), damped Newton and the proposed
saddle-free Newton method (SFN). The hyperparameters of SGD were selected via random
search~\citep{Bergstra+Bengio-2012-small}, and the damping coefficients for the
damped Newton and saddle-free Newton\footnote{Damping is used for numerical stability.} methods were selected from a small set at
each update. 

The theory suggests that the number of saddle points increases exponentially as
the dimensionality of the function increases. From this, we expect that it
becomes more likely for the conventional algorithms such as SGD and Newton methods to
stop near saddle points, resulting in worse performance (on training samples).
Figs.~\ref{fig:results}~(a) and (d) clearly confirm this. With the
smallest network, all the algorithms perform comparably, but as the size grows,
the saddle-free Newton algorithm outperforms the others by a large margin.

A closer look into the different behavior of each algorithm is presented in
Figs.~\ref{fig:results}~(b) and (e) which show the evolution of training error
over optimization. We can see that the proposed saddle-free Newton escapes, or
does not get stuck at all, near a saddle point where both SGD and Newton
methods appear trapped. Especially, at the 10-th epoch in the case of
MNIST, we can observe the saddle-free Newton method rapidly escaping from the
saddle point. Furthermore, Figs.~\ref{fig:results}~(c) and (f) provide evidence that the
distribution of eigenvalues shifts more toward the right as error decreases for all algorithms,
consistent with the theory of random error functions.  The distribution shifts more for SFN,
suggesting it can successfully avoid saddle-points on intermediary error (and large index).

\subsubsection{Effectiveness of saddle-free Newton Method in Deep Neural Networks}

\begin{figure}[t]
\centering
\begin{minipage}{0.49\textwidth}
    \centering
    Deep Autoencoder

\begin{minipage}{0.49\textwidth}
    \centering
    \includegraphics[width=\textwidth,clip=true, trim=0cm 0cm 0cm 0cm]{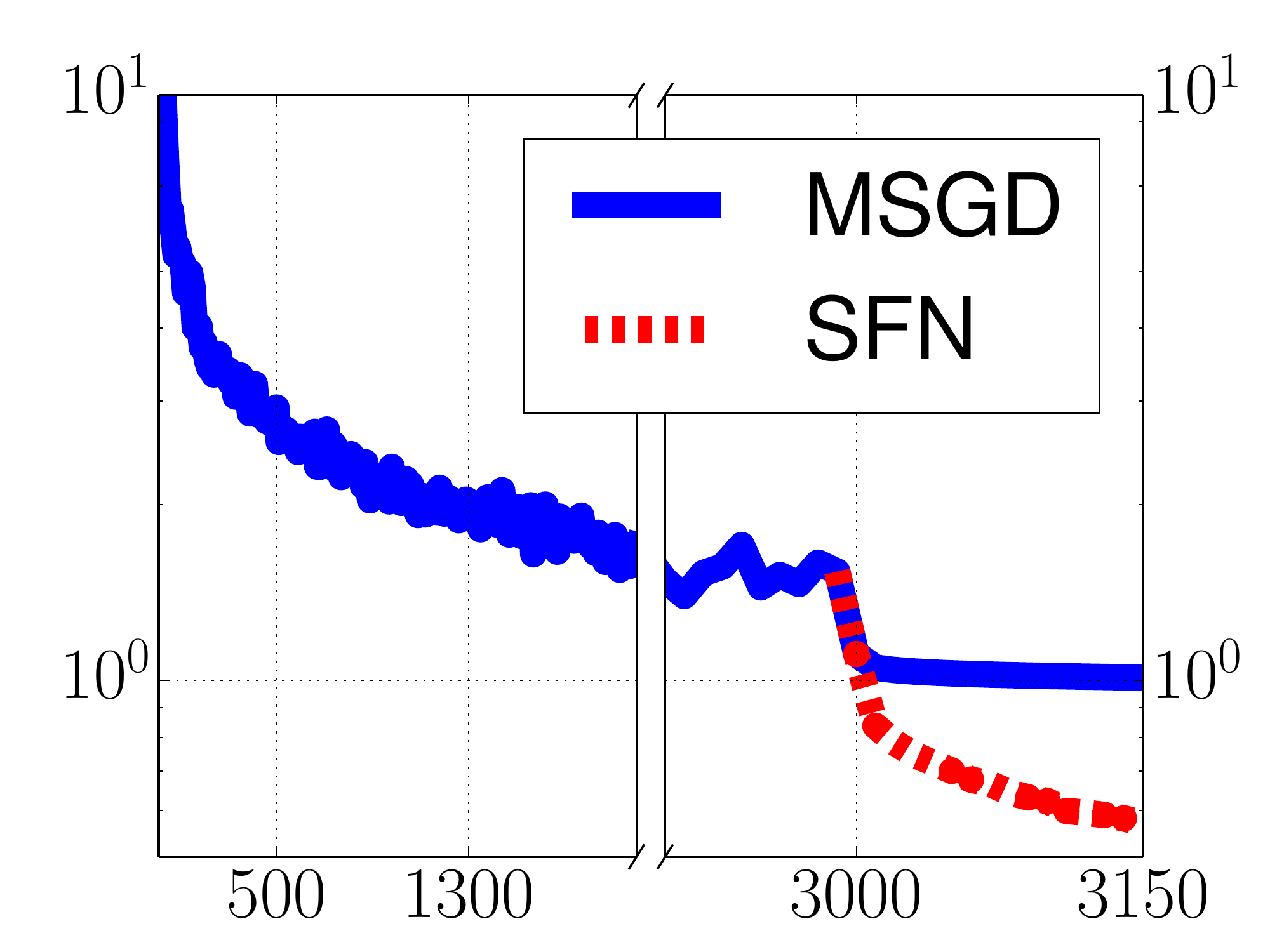}

    (a)
\end{minipage}
\hfill
\begin{minipage}{0.49\textwidth}
    \centering
    \includegraphics[width=\textwidth,clip=true, trim=0cm 0cm 0cm 0cm]{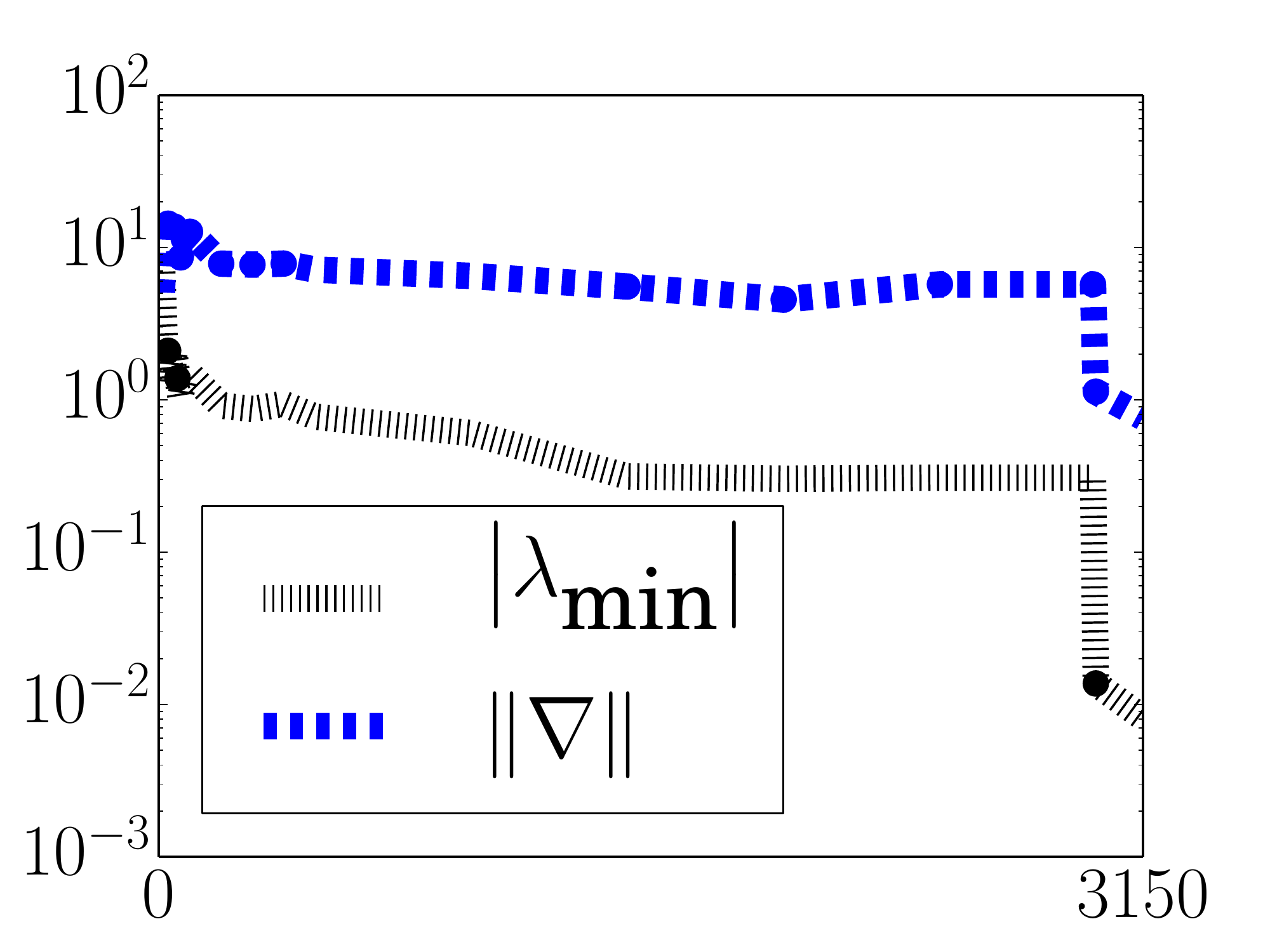}

    (b)
\end{minipage}
\end{minipage}
\hfill
\begin{minipage}{0.49\textwidth}
    \centering
    Recurrent Neural Network

\begin{minipage}{0.49\textwidth}
    \centering
    \includegraphics[width=\textwidth,clip=true, trim=0cm 0cm 0cm
    0cm]{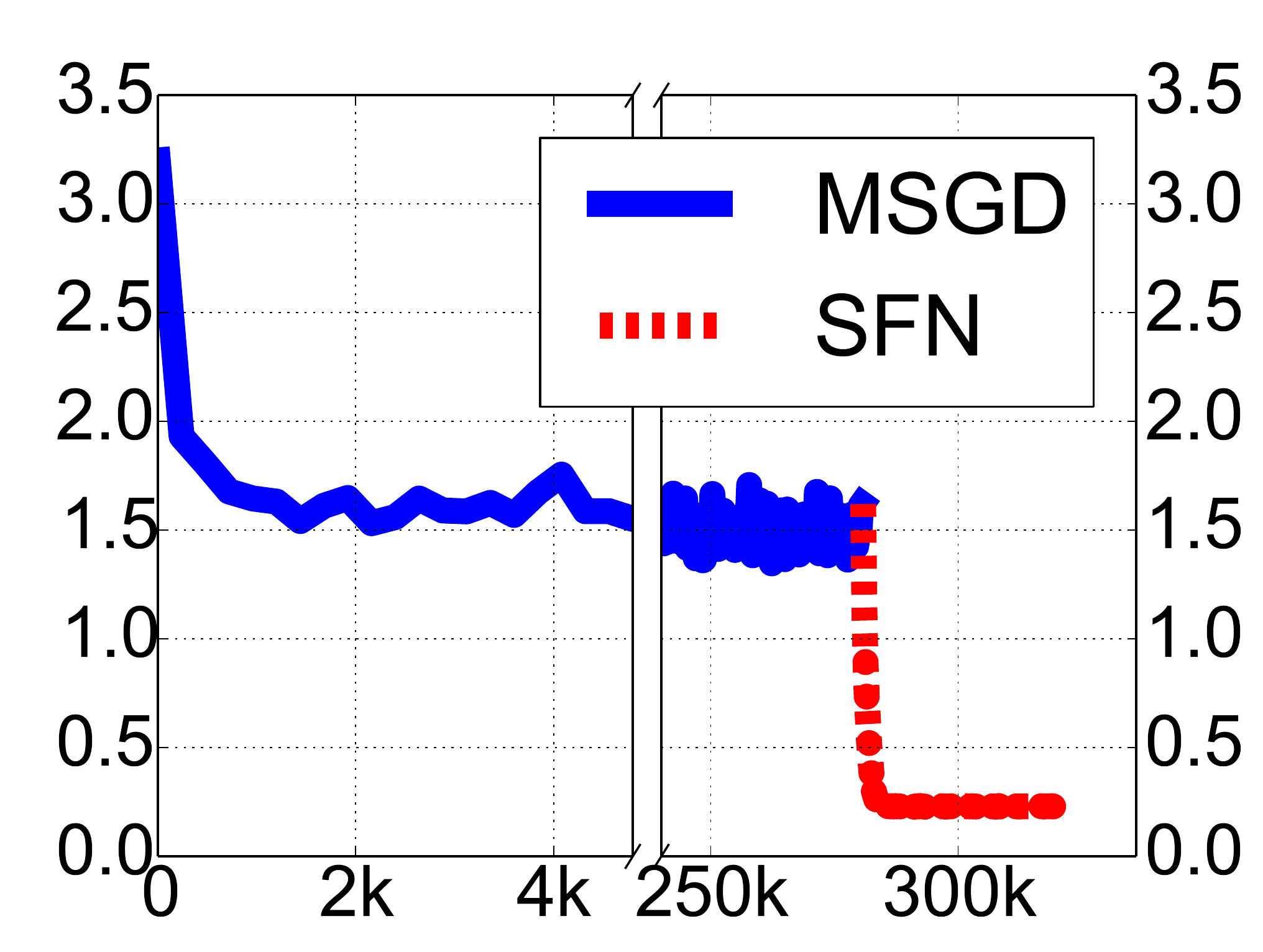}

    (c)
\end{minipage}
\hfill
\begin{minipage}{0.49\textwidth}
    \centering
    \includegraphics[width=\textwidth,clip=true, trim=0cm 0cm 0cm 0cm]{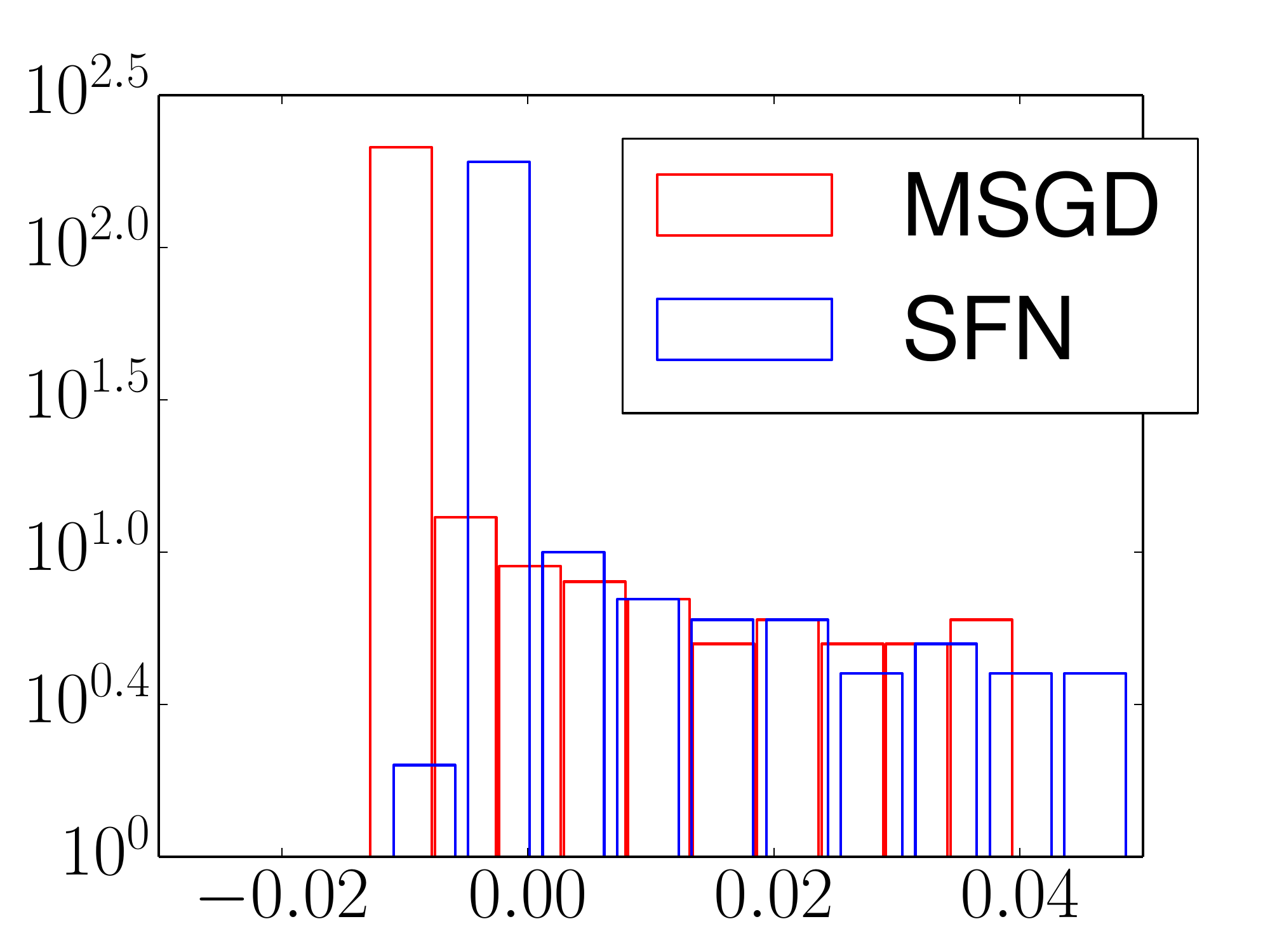}

    (d)
\end{minipage}
\end{minipage}

    \caption{
        Empirical results on training deep autoencoders on MNIST and recurrent
        neural network on Penn Treebank. (a) and (c): The learning curve for SGD and SGD
        followed by saddle-free Newton method. (b) The evolution of the
        magnitude of the most negative eigenvalue and the norm of the gradients
        w.r.t. the number of epochs (deep autoencoder). (d) The distribution of
        eigenvalues of the RNN solutions found by SGD and the SGD continued with
        saddle-free Newton method.
    } 
    \label{fig:sfn} 

    \vspace{-3mm}
\end{figure}

Here, we further show the effectiveness of the proposed saddle-free Newton
method in a larger neural network having seven hidden layers. The neural
network is a deep autoencoder trained on (full-scale) MNIST and considered a
standard benchmark problem for assessing the performance of optimization
algorithms on neural networks~\citep{SutskeverMartensDahlHinton_icml2013}. In
this large-scale problem, we used the Krylov subspace descent approach described
earlier with 500 subspace vectors.

We first trained the model with SGD and observed that learning stalls after
achieving the mean-squared error (MSE) of 1.0. We then continued with the
saddle-free Newton method which rapidly escaped the (approximate) plateau at
which SGD was stuck (See Fig.~\ref{fig:sfn} (a)). Furthermore, even in these
large scale experiments, we were able to confirm that the distribution of Hessian
eigenvalues shifts right as error decreases, and that the proposed saddle-free Newton
algorithm accelerates this shift (See Fig.~\ref{fig:sfn} (b)).

The model trained with SGD followed by the saddle-free Newton method was able to
get the state-of-the-art MSE of $0.57$ compared to the previous best error of
$0.69$ achieved by the Hessian-Free method~\citep{Martens10}. Saddle free Newton method 
does better.  

\subsection{Recurrent Neural Networks: Hard Optimization Problem}

\cho{Recurrent neural networks are widely known to be more difficult to train
    than feedforward neural networks~\citep[see, e.g.,][]{Bengio_trnn94, Pascanu+al-ICML2013-small}}.
    In practice they tend to underfit, and in this section, we want to test 
if the proposed saddle-free Newton method can help avoiding underfitting, assuming that 
that it is caused by saddle points.
We trained a small recurrent neural network having 120 hidden units for the task
of character-level language modeling on Penn Treebank corpus. Similarly to the
previous experiment, we trained the model with SGD until it was clear
that the learning stalled. From there on, training continued with the
saddle-free Newton method. 

In Fig.~\ref{fig:sfn} (c), we see a trend similar to what we observed with the
previous experiments using feedforward neural networks. The SGD stops
progressing quickly and does not improve performance, suggesting that the
algorithm is stuck in a plateau, possibly around a saddle point. As soon as we
apply the proposed saddle-free Newton method, we see that the error drops
significantly.  Furthermore, Fig.~\ref{fig:sfn} (d) clearly shows that the
solution found by the saddle-free Newton has fewer negative eigenvalues,
consistent with the theory of random Gaussian error functions. In addition to the
saddle-free Newton method, we also tried continuing with the truncated Newton
method with damping to see whether if it can improve SGD where it got stuck, but without 
much success.

\section{Conclusion}

In summary, we have drawn from disparate literatures spanning statistical
physics and random matrix theory to neural network theory, to argue that (a)
non-convex error surfaces in high dimensional spaces generically suffer from a
proliferation of saddle points, and (b) in contrast to conventional wisdom
derived from low dimensional intuition, local minima with high error are
exponentially rare in high dimensions.  Moreover, we have provided the first
experimental tests of these theories by performing new measurements of the
statistical properties of critical points in neural network error surfaces.
These tests were enabled by a novel application of Newton's method to search for
critical points of any index (fraction of negative eigenvalues), and they
confirmed the main qualitative prediction of theory that the index of a critical
point tightly and positively correlates with its error level. 

Motivated by this theory, we developed a framework of generalized trust region
methods to search for algorithms that can rapidly escape saddle points.  This
framework allows us to leverage curvature information in a fundamentally
different way than classical methods, by defining the shape of the trust region,
rather than locally approximating the function to second order.  Through further
approximations, we derived an exceedingly simple algorithm, the saddle-free
Newton method, which rescales gradients by the absolute value of the inverse
Hessian.  This algorithm had previously remained heuristic and theoretically
unjustified, as well as numerically unexplored within the context of deep and
recurrent neural networks.   Our work shows that near saddle points it can
achieve rapid escape by combining the best of gradient descent and Newton
methods while avoiding the pitfalls of both. Moreover, through our generalized trust
region approach, our work shows that this algorithm is sensible even far from
saddle points.  Finally, we demonstrate improved optimization on several neural
network training problems.

For the future, we are mainly interested in two directions. The first direction
is to explore methods beyond Kyrylov subspaces, such as one in
\citep{Jascha2014}, that allow the saddle-free Newton method to scale to high
dimensional problems, where we cannot easily compute the entire Hessian matrix.
In the second direction, the theoretical properties of critical points in the
problem of training a neural network will be further analyzed.  More generally,
it is likely that a deeper understanding of the statistical properties of high
dimensional error surfaces will guide the design of novel non-convex
optimization algorithms that could impact many fields across science and
engineering.

\subsubsection*{Acknowledgments}
We would like to thank the developers of
Theano~\citep{bergstra+al:2010-scipy,Bastien-Theano-2012}.  We would also like
to thank CIFAR, and Canada Research Chairs for funding, and Compute Canada, and
Calcul Qu\'ebec for providing computational resources. Razvan Pascanu is
supported by a DeepMind Google Fellowship. Surya Ganguli
thanks the Burroughs Wellcome and Sloan Foundations for support.

\bibliography{myrefs}
\bibliographystyle{natbib}
\newpage
\appendix
\section{Description of the different types of saddle-points}

In general, consider an error function $\LL(\theta)$  where $\theta$ is an $N$
dimensional continuous variable. A critical point is by definition a point
$\theta$ where the gradient of $\LL(\theta)$ vanishes.  All critical points of
$f(\theta)$ can be further characterized by the curvature of the function in
its vicinity, as described by the eigenvalues of the Hessian.  Note that the
Hessian is symmetric and hence the eigenvalues are real numbers. The following
are the four possible scenarios: 

\begin{itemize}
\item If all eigenvalues are non-zero and  positive, then the critical point is
a local minimum.  
\item If all eigenvalues are non-zero and negative, then the critical point is
a local maximum.
\item If the eigenvalues are non-zero and we have both positive and negative
eigenvalues, then the critical point is a saddle point with a \emph{min-max}
structure (see Figure~\ref{fig:different_saddle} (b)). That is, if we restrict
the function $\LL$ to the subspace spanned by the eigenvectors corresponding to
positive (negative) eigenvalues, then the saddle point is a maximum (minimum)
of this restriction.
\item If the Hessian matrix is singular, then the \emph{degenerate} critical
point can be a saddle point, as it is, for example, for $\theta^3, \theta
\in\RR$ or for the monkey saddle (Figure~\ref{fig:different_saddle} (a) and
(c)).  If it is a saddle, then, if we restrict $\theta$ to only change along
the direction of singularity, the restricted function does not exhibit a
minimum nor a maximum; it exhibits, to second order, a plateau. When moving from
one side to other of the plateau, the eigenvalue corresponding to this picked
direction generically changes sign, being exactly zero at the critical point.
Note that an eigenvalue of zero can also indicate the presence of a gutter
structure, a degenerate minimum, maximum or saddle, where a set of connected
points are all minimum, maximum or saddle structures of the same shape
and error.  In Figure~\ref{fig:different_saddle} (d) it is shaped as a circle.
The error function looks like the bottom of a wine bottle, where all points
along this circle are minimum of equal value.
\end{itemize}

A plateau is an almost flat region in some direction.  This structure is given
by having the eigenvalues (which describe the curvature) corresponding to the
directions of the plateau be \emph{close to 0}, but \emph{not exactly 0}. Or,
additionally, by having a large discrepancy between the norm of the
eigenvalues. This large difference would make the direction of ``relative''
small eigenvalues look like flat compared to the direction of large
eigenvalues.

\begin{figure}
    \centering
    \begin{minipage}{0.48\textwidth}
        \centering
        \includegraphics[width=0.85\columnwidth, clip=true, trim=2cm .5cm 2cm .5cm]{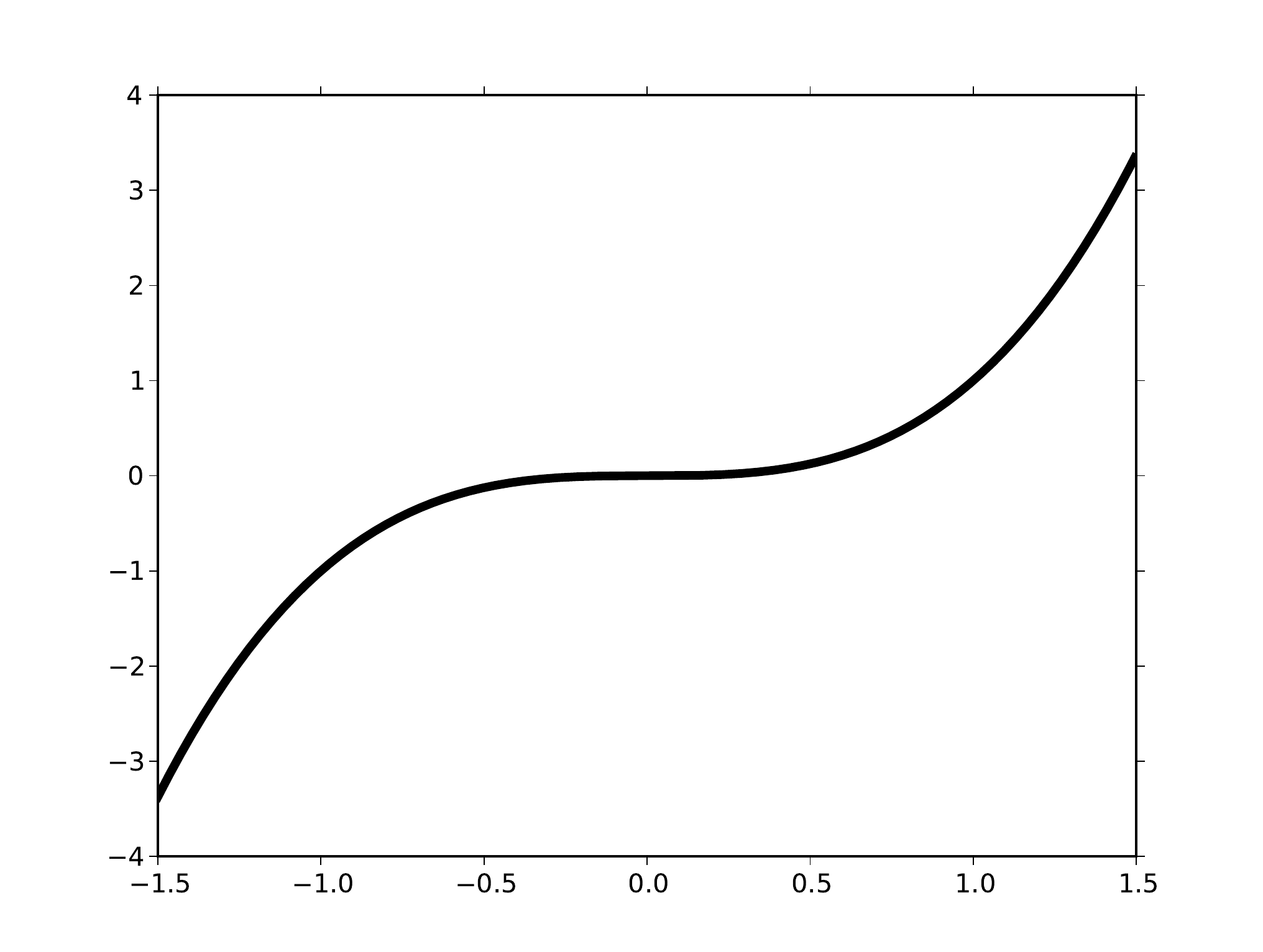}
    \end{minipage}
    \hfill
    \begin{minipage}{0.48\textwidth}
        \centering
        \includegraphics[width=1.\columnwidth, clip=true, trim=2cm 1cm 2cm 1cm]{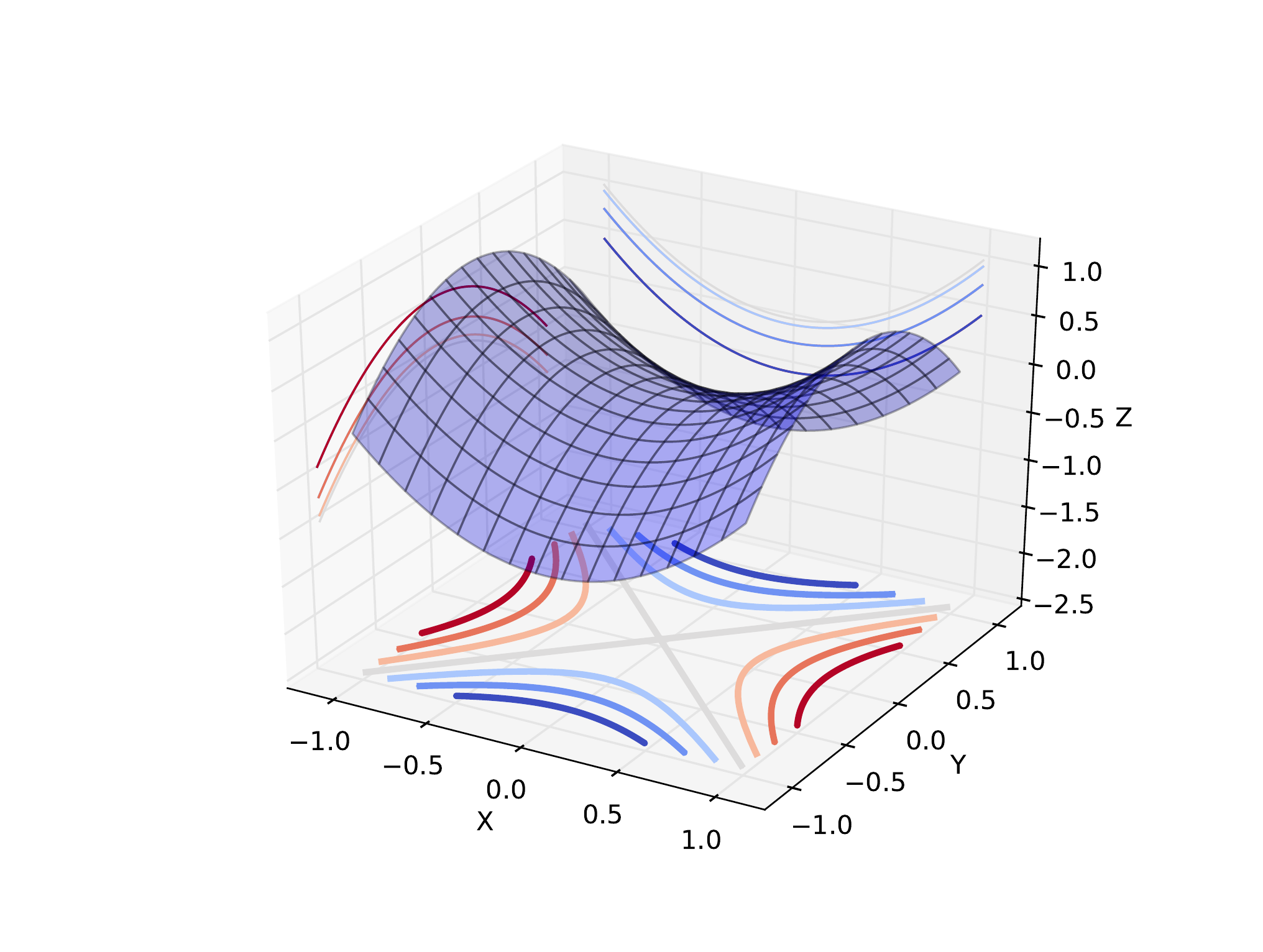}
    \end{minipage}
    \\
    \vspace{3mm}
    \begin{minipage}{0.24\textwidth}
        \centering
        (a) 
    \end{minipage}
    \hfill
    \begin{minipage}{0.24\textwidth}
        \centering
        (b) 
    \end{minipage}
    \\
    \begin{minipage}{0.48\textwidth}
        \centering
        \includegraphics[width=1.\columnwidth, clip=true, trim=2cm 1cm 2cm 1cm]{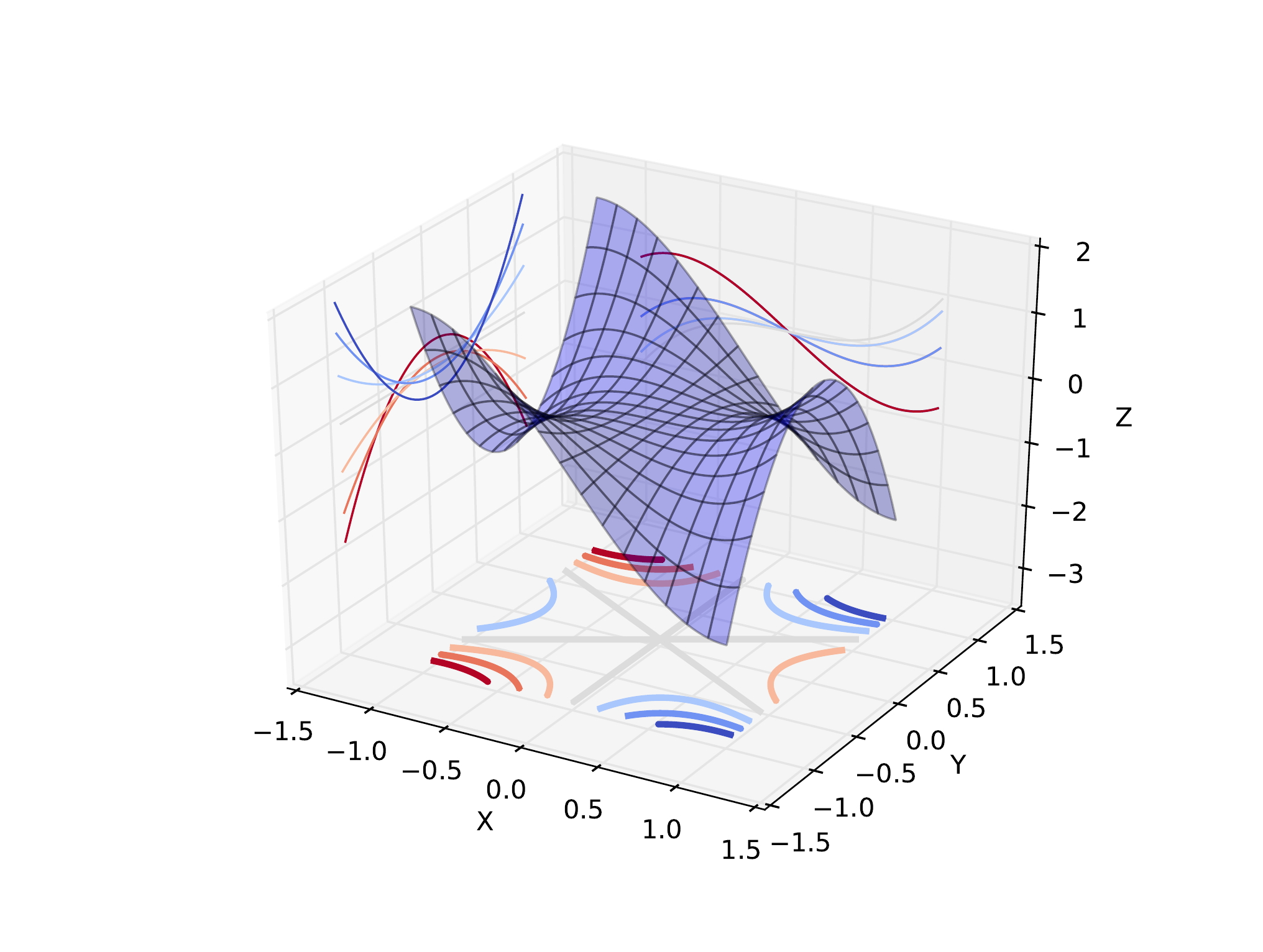}
    \end{minipage}
    \hfill
    \begin{minipage}{0.48\textwidth}
        \centering
        \includegraphics[width=1.\columnwidth, clip=true, trim=2cm 2cm 2cm 2cm]{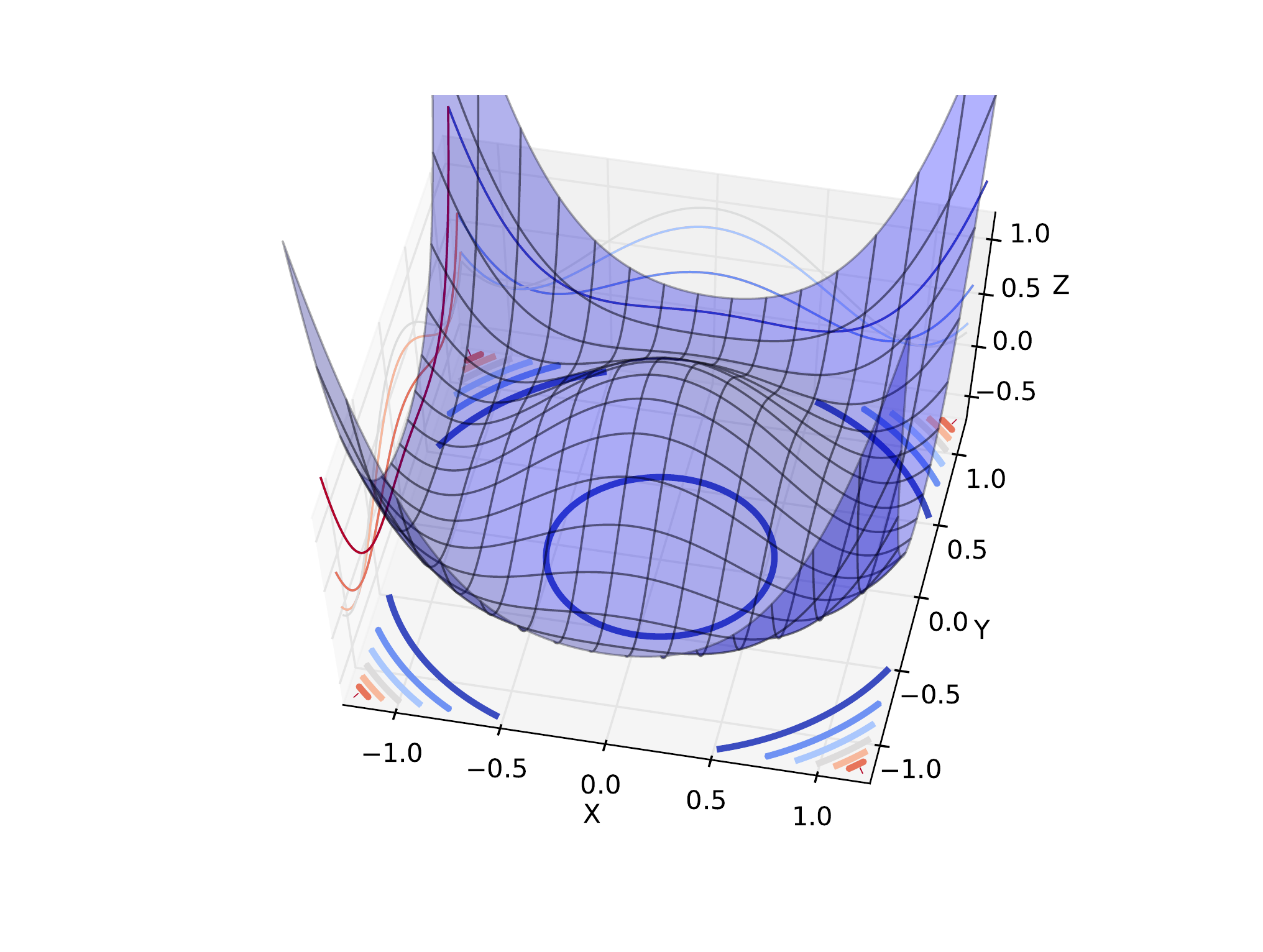}
    \end{minipage}
    \\
    \vspace{3mm}
    \begin{minipage}{0.48\textwidth}
        \centering
        (c) 
    \end{minipage}
    \hfill
    \begin{minipage}{0.48\textwidth}
        \centering
        (d) 
    \end{minipage}

\caption[Different saddle point structures] {Illustrations of three different
types of saddle points (a-c) plus a gutter structure (d).  Note that for the
gutter structure, any point from the circle $x^2 + y^2 = 1$ is a minimum.  The
shape of the function is that of the bottom of a bottle of wine. This means
that the minimum is a ``ring'' instead of a single point. The Hessian is
singular at any of these points. (c) shows a Monkey saddle where you have both
a min-max structure as in (b) but also a 0 eigenvalue, which results, along
some direction, in a shape similar to (a).
}
\label{fig:different_saddle}
\end{figure}

\section{Reparametrization of the space around saddle-points}
\label{sec:apx_reparam}

This reparametrization is given by taking a Taylor expansion of the function
$\LL$ around the critical point. If we assume that the Hessian is not singular,
then there is a neighbourhood around this critical point where this
approximation is reliable and, since the first order derivatives vanish, the
Taylor expansion is given by:

\begin{equation}
\LL(\theta^* + \Dtheta) = \LL(\theta^*) + \frac{1}{2} (\Dtheta) ^\top
\hess \Dtheta
\end{equation}

Let us denote by ${\es\el 1}, \ldots, {\es\el {n_\theta}}$ the eigenvectors of
the Hessian $\hess$ and by ${\lambda\el 1}, \ldots, {\lambda\el {n_\theta}}$
the corresponding eigenvalues.  We can now make a change of coordinates into
the space span by these eigenvectors:

\begin{equation}
\Delta \vv = \frac{1}{2}\left[
\begin{array}{c}
{\es\el{1}}^\top \\
\ldots \\
{\es\el{n_\theta}}^\top 
\end{array}
\right] \Dtheta
\end{equation}

\begin{equation}
\label{eq:new_system_coord}
\LL(\theta^* + \Dtheta) = \LL(\theta^*) + \frac{1}{2}\sum_{i=1}^{n_\theta}
{\lambda\el i} ({\es\el i}^\top \Dtheta)^2 = \LL(\theta^*) + \sum_{i=1}^{n_\theta}
{\lambda\el i} \Delta \vv_i^2 
\end{equation}

\section{Empirical exploration of properties of critical points}
\label{sec:apx_emp}

To obtain the plot on MNIST we used the Newton method to discover nearby critical
points along the path taken by the saddle-free Newton algorithm.  We consider
20 different runs of the saddle-free algorithm, each using a different random
seed. We then run 200 jobs. The first 100 jobs are looking for critical points
near the value of the parameters obtained after some random number of epochs
(between 0 and 20) of a randomly selected run (among the 20 different runs) of
saddle-free Newton method.  To this starting position uniform noise is added of
small amplitude (the amplitude is randomly picked between the different values
$\{10^{-1}, 10^{-2}, 10^{-3}, 10^{-4}\}$ The last 100 jobs look for critical points 
near  uniformally sampled
weights (the range of the weights is given by the unit cube).  
The task (dataset and model) is the same as the one used previously.

To obtain the plots on CIFAR, we have trained multiple 3-layer deep neural
networks using SGD. The activation function of these networks is the tanh
function. We saved the parameters of these networks for each epoch. We trained
100 networks with different parameter initializations between 10 and 300 epochs
(chosen randomly). The networks were then trained using the Newton method to
find a nearby critical point. This allows us to find many different critical
points along the learning trajectories of the networks.

\section{Proof of Lemma~\ref{lemma:constraint}}
\label{sec:apx_proof}

\begin{lemma}
Let $\Abf$ be a nonsingular square matrix in $\RR^{n} \times \RR^{n}$, and
$\example \in \RR^n$ be some vector. Then it holds that $|\example^\top \Abf
\example | \leq \example^\top |\Abf| \example$, where $|\Abf|$ is the matrix
obtained by taking the absolute value of each of the eigenvalues of $\Abf$.
\end{lemma}

\begin{proof}
Let $\es\el 1, \ldots \es\el n$ be the different eigenvectors of $\Abf$ and
$\lambda\el 1, \ldots \lambda\el n$ the corresponding eigenvalues. We now
re-write the identity by expressing the vector $\example$ in terms of these
eigenvalues: 

\begin{align}
|\example^\top \Abf \example|& = \left|\sum_i (\example^\top \es\el i){\es\el{i}}^\top \Abf \example \right| 
 = \left|\sum_i (\example^\top \es\el i)\lambda\el i ({\es\el  i}^\top \example) \right| 
 = \left|\sum_i \lambda\el i (\example^\top \es\el i)^2 \right| \nonumber
\end{align}

We can now use the triangle inequality $|\sum_i x_i| \leq \sum_i |x_i|$ and get
that 

\begin{align}
|\example^\top \Abf \example|& \leq \sum_i |(\example^\top \es\el i)^2 \lambda\el i| 
= \sum_i (\example^\top \es\el i)|\lambda\el i| ({\es\el i}^\top\example) 
 = \example^\top |\Abf | \example \nonumber
\end{align}

\end{proof}

\section{Implementation details for approximate saddle-free Newton}

The Krylov subspace is obtained through a slightly modified Lanczos process
(see Algorithm \ref{alg:lanczos}). The initial vector of the algorithm is
the gradient of the model. As noted by \citet{VinyalsAISTATS12}, we found it
was useful to include the previous search direction as the last vector of the
subspace.

As described in the main paper, we have $\frac{\partial \hat{f}}{\partial
\alpha} = {\bf V} \left(\frac{\partial f}{\partial \theta}\right)^\top$ and 
$\frac{\partial^2 \hat{f}}{\partial \alpha^2} = {\bf V} \left(\frac{\partial^2
f}{\partial \theta^2}\right) {\bf V}^\top$. Note that the calculation of the
Hessian in the subspace can be greatly sped up by memorizing the vectors
${\bf \mathbf{V}}_{i} \frac{\partial^2 f}{\partial \theta^2}$ during the Lanczos
process. Once memorized, the Hessian is simply the product of the two matrices
${\bf V}$ and ${\bf \mathbf{V}}_{i} \frac{\partial^2 f}{\partial \theta^2}$.

We have found that it is beneficial to perform multiple optimization steps
within the subspace. We do not recompute the Hessian for these steps under the
assumption that the Hessian will not change much.

\begin{algorithm}[H]
\caption{Obtaining the Lanczos vectors}
\label{alg:lanczos}
\begin{algorithmic}
    \REQUIRE ${\bf g} \gets -\frac{\partial f}{\partial \theta}$
    \REQUIRE ${\Delta \theta}$ \text{(The past weight update)}
    
    \STATE ${\bf \mathbf{V}}_0 \gets 0$
    \STATE ${\bf \mathbf{V}}_1 \gets \frac{{\bf g}}{\| {\bf g} \|} $
    \STATE $\beta_1 \gets 0$
    
    \FOR{$i = 1 \to k - 1$}
    \STATE ${\bf w}_i \gets {\bf \mathbf{V}}_{i} \frac{\partial^2 f}{\partial \theta^2} \text{(Implemented efficient through L-Op \citep{Pearlmutter94fastexact})}$ 
    
    \IF{$i = k - 1$}
        \STATE ${\bf w}_i \gets \Delta \theta$
    \ENDIF
    
    \STATE $\alpha_i \gets {\bf w}_i {\bf \mathbf{V}}_{i}$
    
    \STATE ${\bf w}_i \gets {\bf w}_i - \alpha_i {\bf \mathbf{V}}_{i} - \beta_i {\bf \mathbf{V}}_{i-1}$
    
    \STATE $\beta_{i+1} \gets \|{\bf w}_i\|$
    \STATE ${\bf \mathbf{V}}_{i+1} \gets \frac{\bf w}{\|{\bf w}_i\|}$
    
    \ENDFOR
\end{algorithmic}
\end{algorithm}

\section{Experiments}

\subsection{Existence of Saddle Points in Neural Networks}

For feedforward networks using SGD, we choose the following 
hyperparameters using the random search strategy~\citep{Bergstra+Bengio-2012-small}:
\begin{itemize}
    \item Learning rate
    \item Size of minibatch
    \item Momentum coefficient
\end{itemize}
For random search, we draw 80 samples and pick the best one.

For both the Newton and saddle-free Newton methods, the damping coefficient is
chosen at each update, to maximize the improvement, among $\left\{ 10^0,
    10^{-1}, 10^{-2}, 10^{-3}, 10^{-4}, 10^{-5} \right\}$.

\subsection{Effectiveness of saddle-free Newton Method in Deep Neural Networks}

The deep auto-encoder was first trained using the protocol used by
\citet{SutskeverMartensDahlHinton_icml2013}. In these experiments we use
classical momentum.

\subsection{Recurrent Neural Networks: Hard Optimization Problem}

We initialized the recurrent weights of RNN to be orthogonal as suggested by
\citet{Saxe-ICLR2014}. The number of hidden units of RNN is fixed to 120. For
recurrent neural networks using SGD, we choose the following hyperparameters
using the random search strategy:

\begin{itemize}
    \item Learning rate
    \item Threshold for clipping the gradient~\citep{Pascanu+al-ICML2013-small}
    \item Momentum coefficient
\end{itemize}

For random search, we draw 64 samples and pick the best one. Just like in the
experiment using feedforward neural networks, the damping coefficient of both
the Newton and saddle-free Newton methods was chosen at each update, to maximize
the improvement.

We clip the gradient and saddle-free update step if it exceeds certain threshold
as suggested by \citet{Pascanu+al-ICML2013-small}.

Since it is costly to compute the exact Hessian for RNN's, we used the
eigenvalues of the Hessian in the Krylov subspace to plot the distribution of
eigenvalues for Hessian matrix in Fig.~\ref{fig:sfn} (d).

\end{document}